%% file: main.tex
\documentclass[journal]{IEEEtran}

\usepackage{cite}
\usepackage[final]{graphicx}
\usepackage{amssymb}
\usepackage{subcaption}
\usepackage{epstopdf}
\usepackage{amsmath,amssymb,amsthm}
\usepackage{color}
\usepackage{braket}
\usepackage{enumitem}
\usepackage{url}

\usepackage{fancyhdr}
\usepackage{tabulary}
\usepackage{multirow}

\input{macros}

\begin{document}

%

%

\sloppy

\title{A Sampling Theory Perspective of Graph-based Semi-supervised Learning}

\author{Aamir~Anis,~\IEEEmembership{Student~Member,~IEEE,}
        Aly~El~Gamal,~\IEEEmembership{Member,~IEEE,}
        Salman~Avestimehr,~\IEEEmembership{Senior~Member,~IEEE,}
        and~Antonio~Ortega,~\IEEEmembership{Fellow,~IEEE}
\thanks{This work is supported in part by NSF under grants CCF-1410009, CCF-1527874, CCF-1408639, NETS-1419632 and by AFRL and DARPA under grant 108818.}%
\thanks{S. Avestimehr and A. Ortega are with the Ming Hsieh Department
of Electrical Engineering, University of Southern California. A. Anis is currently with Google Inc., he was affiliated with the University of Southern California at the time this work was completed. A. El Gamal is with the Department of Electrical and Computer Engineering,
Purdue University.}
\thanks{E-mail: aanis@google.com, avestimehr@ee.usc.edu, ortega@sipi.usc.edu, elgamala@purdue.edu.}
\thanks{Copyright (c) 2017 IEEE. Personal use of this material is permitted.  However, permission to use this material for any other purposes must be obtained from the IEEE by sending a request to pubs-permissions@ieee.org.}
}


\maketitle
%
%
%
\begin{abstract}
Graph-based methods have been quite successful in solving unsupervised and semi-supervised learning problems, as they provide a means to capture the underlying geometry of the dataset. It is often desirable for the constructed graph to satisfy two properties: first, data points that are similar in the feature space should be strongly connected on the graph, and second, the class label information should vary smoothly with respect to the graph, where smoothness is measured using the spectral properties of the graph Laplacian matrix. Recent works have justified some of these smoothness conditions by showing that they are strongly linked to the semi-supervised smoothness assumption and its variants. In this work, we reinforce this connection by viewing the problem from a graph sampling theoretic perspective, where class indicator functions are treated as bandlimited graph signals (in the eigenvector basis of the graph Laplacian) and label prediction as a bandlimited reconstruction problem. Our approach involves analyzing the bandwidth of class indicator signals generated from statistical data models with separable and nonseparable classes. These models are quite general and mimic the nature of most real-world datasets. Our results show that in the asymptotic limit, the bandwidth of any class indicator is also closely related to the geometry of the dataset. This allows one to theoretically justify the assumption of bandlimitedness of class indicator signals, thereby providing a sampling theoretic interpretation of graph-based semi-supervised classification.
\end{abstract}
%

%
\section{Introduction}
%
%

The abundance of unlabeled data in various machine learning applications, along with the prohibitive cost of labeling, has led to growing interest in semi-supervised learning. This paradigm deals with the task of classifying data points in the presence of very little labeling information by relying on the geometry of the dataset. Assuming that the features are well-chosen, a natural assumption in this setting is 
to consider the marginal density $p(\xv)$ of the feature vectors to be informative about the labeling function $f(\xv)$ defined on the points. This assumption is fundamental to the semi-supervised learning problem both in the classification and the regression settings, and is also known as the \emph{semi-supervised smoothness assumption}~\cite{chapelle_book}, which states that the label function is smoother in regions of high data density. There also exist other similar variants of this assumption specialized for the classification setting, namely, the \emph{cluster assumption}~\cite{zhou_nips_04} (points in a cluster are likely to have the same class label) or the \emph{low density separation assumption}~\cite{narayanan_06} (decision boundaries pass through regions of low data density). Most present day algorithms for semi-supervised learning rely on one or more of these assumptions to predict the unknown labels.

In practice, graph-based methods have been found to be quite suitable for geometry-based learning tasks, primarily because they provide an easy way of exploiting information from the geometry of the dataset. These methods 
involve constructing a distance-based similarity graph whose vertices (nodes) represent the data points and whose edge weights are in general a decreasing function of the distances between them.
The learning task then involves predicting the labels of the unknown nodes, given the known labels, often called the transductive learning paradigm.
The key assumption here is that the label function is ``smooth'' over the graph, in the sense that labels of vertices do not vary much over edges with high weights (i.e., edges that connect close or similar points).
%
%
There are numerous ways of quantitatively imposing smoothness constraints over label functions defined on the vertices of a similarity graph.
Most graph-based semi-supervised classification algorithms incorporate one of these criteria as a penalty against the fitting error in a regularization problem, or as a constraint term while minimizing the fitting error in an optimization problem.
For example, a commonly used measure of smoothness for a label function $\fv$ is the graph Laplacian regularizer $\fv^T \Lm \fv$ ($\Lm$ being the graph Laplacian), and many algorithms involve minimizing this quadratic energy function while ensuring that $\fv$ satisfies the known set of labels~\cite{zhu_03, zhou_nips_04}. Another example is the graph total variation~\cite{trillos_arma_16}. There also exist higher-order variants of the smoothness measure such as iterated graph Laplacian regularizers $\fv^T \Lm^m \fv$~\cite{zhou_aistats_11} and the $p$-Laplacian regularizer~\cite{buhler_icml_09,elalaoui_colt_16}, that have been shown to make the problem more well-behaved.
On the other hand, a spectral theory based classification algorithm restricts $\fv$ to be spanned by the first few eigenvectors of the graph Laplacian~\cite{belkin_nips_02,belkin_ml_04}, that are known to form a representation basis for smooth functions on the graph.
In each of the examples, the criterion enforces smoothness of the labels over the graph -- a lower value of the regularizer $\fv^T \Lm \fv$, and a smaller number of leading eigenvectors to model $\fv$ imply that vertices that are close neighbors on the graph are more likely to have the same label.

A more recent approach, derived from Graph Signal Processing (GSP)~\cite{shuman_spm_13}, considers the semi-supervised learning problem from the perspective of sampling theory for graph signals~\cite{anis_icassp_14, shomorony_globalsip_14, chen_tsp_15, anis_tsp_16}. It involves treating the class label function $\fv$ as a \emph{bandlimited} graph signal, and label prediction as a \emph{bandlimited reconstruction} problem. 
The advantage of this approach is that one can also analyze, using sampling theory, the \emph{label complexity} of graph-based semi-supervised classification, that is, the fraction of labeled vertices on the graph required for predicting the labels of the unlabeled vertices.
A key ingredient in this formulation is the \emph{bandwidth} $\omega(\fv)$ of signals on the graph, which is defined as the largest Laplacian eigenvalue for which the projection of the signal over the corresponding eigenvector is non-zero. Signals with lower bandwidth tend to be smoother on the graph and have a lower label complexity. Label prediction using bandlimited reconstruction then involves estimating a graph signal that minimizes prediction error on the labeled set under a bandwidth constraint. This can also be carried out without explicitly computing the eigenvectors of the Laplacian, and has been shown to be quite competitive in comparison to state-of-the-art graph-based semi-supervised learning methods~\cite{gadde_kdd_14}.

Although graph-based semi-supervised learning methods are well-motivated, their connection to the underlying geometry of the dataset had not been clearly understood so far in a theoretical sense. Recent works focused on justifying these approaches by exploring their geometrical interpretation in the limit of infinitely available unlabeled data. This is typically done by assuming a probabilistic generative model for the dataset and analyzing the graph smoothness criteria in the asymptotic setting for certain commonly-used graph construction schemes.
For example, it has been shown that for data points drawn from a smooth distribution with an associated smooth label function (i.e., the regression setting), the graph Laplacian-based regularizers converge in the limit of infinite data points to some density-weighted variational energy functional that penalizes large variations of the labels in high density regions~\cite{belkin_ml_04,bousquet_04,hein_colt_06,belkin_jcss_08,zhou_aistats_11,zhou_kdd_11,trillos_arma_16,trillos_arxiv_16,slepcev_arxiv_17}. A similar connection ensues for semi-supervised learning problems in the classification setting (i.e., when labels are discrete in the feature space). If points drawn from a smooth distribution are separated by a smooth boundary into two classes, then the graph cut for the partition converges to a weighted volume of the boundary~\cite{narayanan_06,castro_aap_12,maier_13,trillos_jmlr_16}. This is consistent with the low density separation assumption -- a low value of the graph cut implies that the boundary passes through regions of low data density.

To our knowledge, no such connections have been drawn for the sampling theoretic approach to learning. A geometrical interpretation of this approach would help complete our theoretical understanding of graph-based semi-supervised learning approaches and strengthen their link with the semi-supervised smoothness assumption and its variants.
Therefore, in this work, we seek answers for the following questions:
\begin{itemize}[leftmargin=10pt]
\item What is the connection between the bandwidth of class indicator signals over the similarity graph and the underlying geometry of the data set?
\item What is the interpretation of the bandlimited reconstruction approach for label prediction?
\item How many labeled examples does one require for predicting the unknown labels?
\end{itemize}
To answer these questions, our work analyzes the asymptotic behavior of an iterated Laplacian-based bandwidth estimator for class indicator signals on similarity graphs constructed from a statistical model for the feature vectors. 
To make our analysis as general as possible, we consider two data models: \emph{separable} and \emph{nonseparable}. These generative models are quite practical and can be used to mimic most datasets in the real world. The separable model assumes that data points are independently drawn from an underlying probability distribution in the feature space and each class is separated from the others by a smooth boundary. On the other hand, the nonseparable model assumes a mixture distribution for the data where the data points are drawn independently with certain probability from separate class conditional distributions. 
We also introduce a notion of ``boundaries'' for classes in the nonseparable model in the form of \emph{overlap} regions (i.e., the region of ambiguity), defined as the set of points where the probability of belonging and not belonging to a class are both non-zero. This definition is quite practical and useful for characterizing the geometry of such datasets.

Using the data points, we consider a specific graph construction scheme that applies the Gaussian kernel over Euclidean distances between feature vectors for computing their similarities (our analysis can be generalized easily to arbitrary kernels under simple assumptions). In order to compute the bandwidth of any signal on the graph, we define an estimator based on the iterated Laplacian regularizer. A significant portion of this paper focuses on analyzing the stochastic convergence of this bandwidth estimate (using variance-bias decomposition) in the limit of infinite data points for any class indicator signal on the graph. 
The analysis in our work suggests a novel sampling theoretic interpretation of graph-based semi-supervised learning and the main contributions can be summarized as follows:
\begin{itemize}[leftmargin=10pt]
\item \emph{Relationship between bandwidth and data geometry.} 
For the separable model, we show that under certain rate conditions, the bandwidth estimate for any class indicator signal over the graph converges to the supremum of the data density over the class boundary. Similarly, for the nonseparable model, we show that the bandwidth estimate converges to the supremum of the density over the overlap region. Based on these results, we conjecture, with supporting experiments, that the bandwidths also converge to the same values.
\item \emph{Interpretation of bandlimited reconstruction.} 
Using the geometrical interpretation of the bandwidth, we conclude that bandlimited reconstruction allows one to choose the complexity of the hypothesis space while predicting unknown labels (i.e., a larger bandwidth allows more complex class boundaries).
\item \emph{Quantification of label complexity for sampling theory-based learning.} 
For both the separable and nonseparable models, we conjecture, with supporting arguments and experiments, that the fraction of labeled nodes on the graph for reconstructing class indicator signals converges, in the asymptotic limit, to the probability mass of the sublevel set that entirely encompasses the boundary.
\end{itemize}
Our analysis has significant implications: Firstly, class indicator signals have a low bandwidth if class boundaries lie in regions of low data densities, that is, the semi-supervised assumption holds for graph-based methods. And secondly, our analysis also helps quantify the impact of bandwidth and data geometry in semi-supervised learning problems. Specifically, it enables us to theoretically assert that for the sampling theoretic approach to graph-based semi-supervised learning, the label complexity of class indicator signals over the graph is indeed lower if the boundary lies in regions of low data density, as demonstrated empirically in earlier works~\cite{belkin_nips_02,belkin_ml_04}.

The rest of this paper is organized as follows: In Section~\ref{sec:prelim}, we formally introduce the statistical data models and the graph construction scheme for analysis, along with a precursor of concepts from graph sampling theory. In Section~\ref{sec:related_work}, we review prior work and underline their connections with our work. In Section~\ref{sec:main_results}, we state our main results and outline their implications. In Section~\ref{sec:proofs}, we prove the major building blocks for our results. We finally conclude with numerical validation in Section~\ref{sec:exp}, followed by discussion and an outline of future work in Section~\ref{sec:end}. It is worth noting that the bandwidth convergence result for the separable model and an interpretation of bandlimited reconstruction were given in our preliminary work~\cite{anis_15}. This paper presents complete formal proofs for those results, extends them to the nonseparable model, and also analyzes label complexity. 



\section{Preliminaries}
\label{sec:prelim}
%
%
\subsection{Data models}
%
%
\subsubsection{The separable model}
\begin{figure*}
\begin{center}
\begin{subfigure}{0.4\textwidth}
\centering
\includegraphics[width=0.9\linewidth]{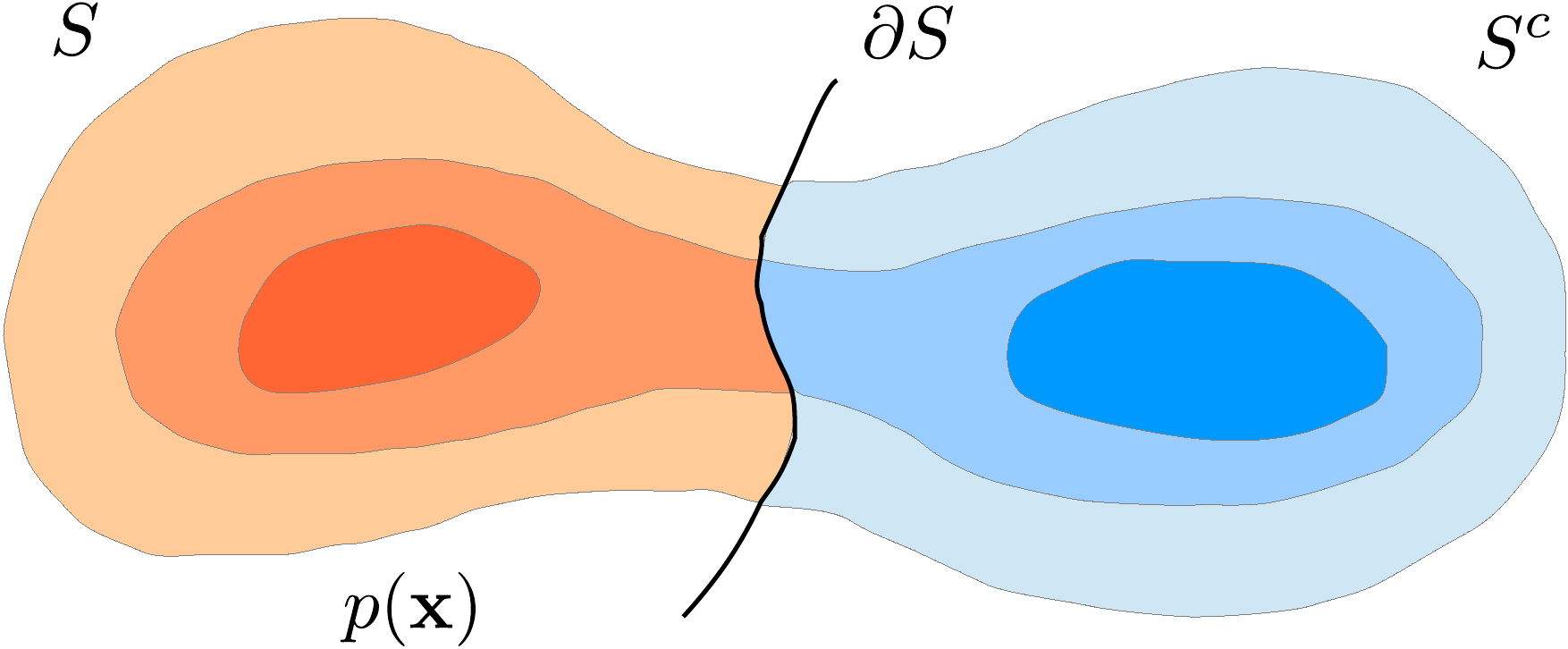}
\caption{}
\label{fig:sep}
\end{subfigure}
\qquad
\begin{subfigure}{0.4\textwidth}
\centering
\includegraphics[width=0.9\linewidth]{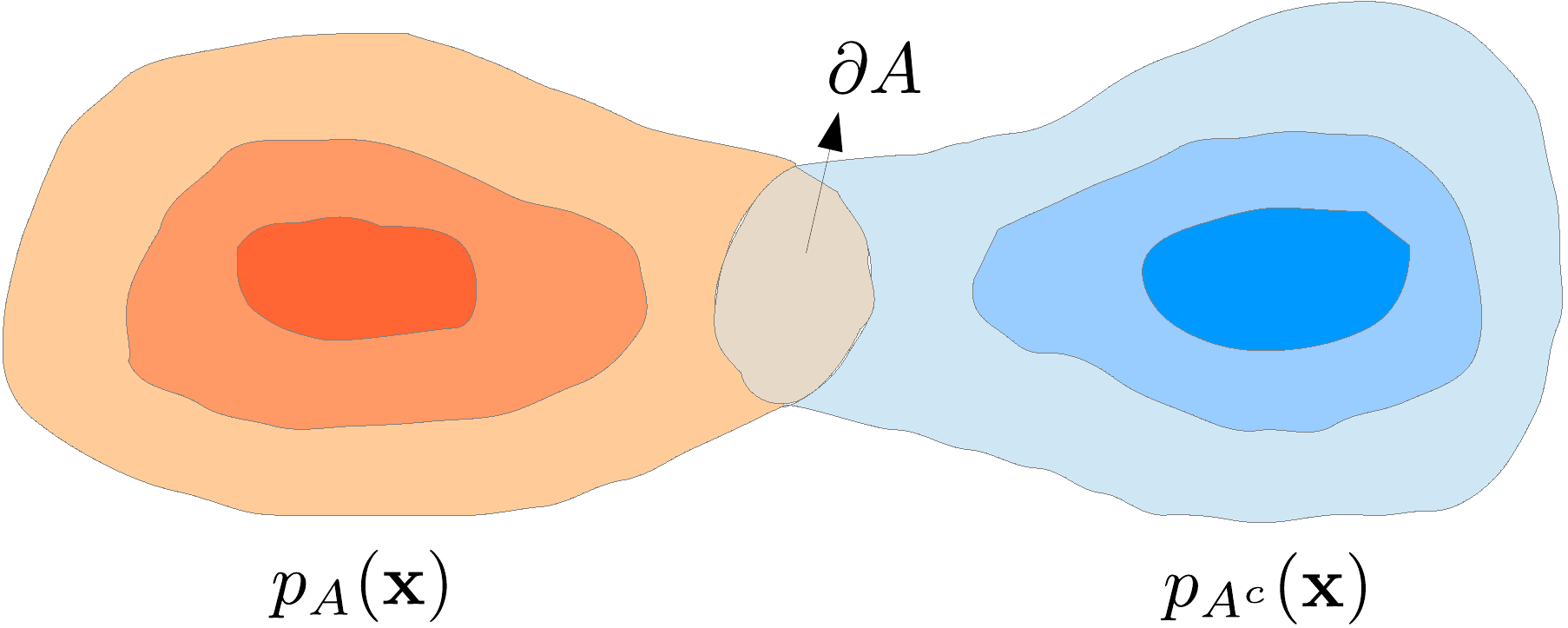}
\caption{}
\label{fig:nonsep}
\end{subfigure}
\end{center}
\caption{Statistical models of data considered in this work: (a) The separable model, (b) The nonseparable model. Darker shades indicate regions of higher density.}
\end{figure*}

In this model, we assume that the dataset consists of a pool of $n$ random, $d$-dimensional feature vectors $\Xc = \{\Xm_1, \Xm_2, \dots, \Xm_n\}$ drawn independently from some probability density function $p(\xv)$ supported on $\mathbb{R}^d$ (this is assumed for simplicity, the analysis can be extended to subsets $D \subset \mathbb{R}^d$ and low-dimensional manifolds $\mathcal{M}$ in $\mathbb{R}^d$, but would more technically involved). To simplify our analysis, we also assume that $p(\xv)$ is bounded from above, Lipschitz continuous and twice differentiable.
We assume that a smooth hypersurface $\partial S$, with radius of curvature lower bounded by a constant $\tau$, splits $\mathbb{R}^d$ into two disjoint classes $S$ and $S^c$, with indicator functions $1_{S}(\xv):\mathbb{R}^d \rightarrow \{0,1\}$ and $1_{S^c}(\xv):\mathbb{R}^d \rightarrow \{0,1\}$.
This is illustrated in Figure~\ref{fig:sep}.
Thus, the $n$-dimensional class indicator signal for class $S$ is denoted by the bold-faced vector notation $\onev_S \in \{0,1\}^n$, and defined as $(\onev_S)_i := 1_{S}(\Xm_i)$, i.e., the $i^\text{th}$ entry of $\onev_S$ is $1$ if $\Xm_i \in S$ and $0$ otherwise.
\subsubsection{The nonseparable model}
In this model, we assume that each class has its own conditional distribution supported on $\mathbb{R}^d$ (that may or may not overlap with other distributions of other classes). The data set consists of a pool of $n$ random and independent $d$-dimensional feature vectors $\Xc = \{\Xm_1, \Xm_2, \dots, \Xm_n\}$ drawn independently from any of the distributions $p_i(\xv)$ with probabilities $\alpha_i$, such that $\sum_i \alpha_i = 1$.
For our analysis, we consider a class denoted by an index $A$ with selection probability $\alpha_A$, class conditional distribution $p_A(\xv)$ and an $n$-dimensional indicator vector $\onev_A$ whose $i^\text{th}$ component takes value $1$ if $\Xm_i$ is drawn from class $A$. Note that $\onev_A$ does not have a continuous domain counterpart, unlike $\onev_S$ which is sampled from the indicator function $1_S(\xv)$ on points in $\Xc$. We illustrate the nonseparable model in Figure~\ref{fig:nonsep}. Further, we denote by $\alpha_{A^c} = 1- \alpha_A$ the probability that a point does not belong to $A$ and by $p_{A^c}(\xv) = \sum_{i \neq A}\alpha_i p_i(\xv) / \alpha_{A^c}$ the density of all such points. The marginal distribution of data points is then given by the mixture density 
\begin{equation}
p(\xv) = \alpha_A p_A(\xv) + \alpha_{A^c} p_{A^c}(\xv).
\end{equation}
Once again, to simplify our analysis, we assume that all distributions are Lipschitz continuous, bounded from above and twice differentiable in $\mathbb{R}^d$. Next, we introduce the notion of a ``boundary'' for classes in the nonseparable model as follows: for class $A$, we define its \emph{overlap} region $\partial A$ as
\begin{equation}
\partial A := \{ \xv \in \mathbb{R}^d \;|\; p_A(\xv) p_{A^c}(\xv) > 0 \}.
\end{equation}
Intuitively, $\partial A$ can be considered as the region of ambiguity, where both points belonging and not belonging to $A$ co-exist. In other words, $\partial A$ can be thought of as a ``boundary'' that separates the region where points can only belong to $A$ from the region where points can never belong to $A$. Since class indicator signals on graphs will change values only within the overlap region, one would expect that the indicators will be smoother if there are fewer data points within this region. We shall show later that this is indeed the case, both theoretically and experimentally. Note that the definition of the boundary is not very meaningful for class conditional distributions with decaying tails, such as the Gaussian, since the boundary in this case technically encompasses the entire feature space. However, in such cases, one can approximate the boundary with appropriate thresholds in the definition and this approximation can also be formalized for distributions with exponentially decaying tails.
\subsection{Graph construction} 
Using the $n$ feature vectors, we construct an undirected distance-based similarity graph where nodes represent the data points and edge weights are proportional to their similarity, given by the Gaussian kernel:
\begin{equation}
w_{ij} := K_{\sigma^2}(\Xm_i,\Xm_j) = \frac{1}{(2\pi\sigma^2)^{d/2}} e^{-\|\Xm_i - \Xm_j\|^2/2\sigma^2 },
\label{eq:graph}
\end{equation}
where $\sigma$ is the variance (bandwidth) of the Gaussian kernel.
Further, we assume $w_{ii} = 0$, i.e., the graph does not have self-loops. 
The adjacency matrix of the graph $\Wm$ is an $n\times n$ symmetric matrix with elements $w_{ij}$, while the degree matrix is a diagonal matrix with elements $\Dm_{ii} = \sum_j w_{ij}$. 
We define the graph Laplacian as $\Lm = \frac{1}{n}(\Dm-\Wm)$. Normalization by $n$ ensures that the norm of $\Lm$ is stochastically bounded as $n$ grows. Since the graph is undirected, $\Lm$ is a symmetric matrix with non-negative eigenvalues $0 \leq \lambda_1 \leq \dots \leq \lambda_n$ and an orthogonal set of corresponding eigenvectors $\{\uv_1, \dots, \uv_n\}$. It is known that for a larger eigenvalue $\lambda$, the corresponding eigenvector $\uv$ exhibits greater variation when plotted over the nodes of the graph~\cite{shuman_spm_13}. Thus, one of the fundamental postulates of Graph Signal Processing consists of using the eigen-decomposition of $\Lm$ to provide a notion of frequency for graph signals, with the eigenvalues acting as \emph{graph frequencies} and the eigenvectors forming the \emph{graph Fourier basis}~\cite{shuman_spm_13}.

\subsection{Graph sampling theory: bandwidth, bandlimited reconstruction and label complexity}
\label{subsec:sampling_theory}
In traditional sampling theory, bandwidth plays an important role in specifying the inherent dimensionality of a signal and therefore determines the sampling rate required for perfect reconstruction.
A similar notion exists for signals defined over graphs -- the bandwidth $\omega(\fv)$ of any signal $\fv$ on the graph is defined as the largest eigenvalue for which the projection of the signal on the corresponding eigenvector is non-zero~\cite{narang_icassp13,anis_icassp_14,anis_tsp_16}, i.e.,
\begin{equation}
\omega(\fv) := \max_i \big\{ \lambda_i \;\big|\; |\uv_i^T \fv| > 0 \big\}.
\end{equation}
Signals with lower bandwidth have low frequency content, and tend to be smoother on the graph.

Bandwidth plays a central role in the sampling theoretic approach to semi-supervised learning, where the class indicator signals are assumed to be bandlimited over the similarity graph and interpolated through bandlimited reconstruction. For a ground-truth signal $\fv$ that we are trying to reconstruct, and whose values are known only on a subset $L \subset \{1,2,\dots,n\}$, this approach involves solving the following least-squares problem~\cite{narang_globalsip13,gadde_kdd_14}:
\begin{equation}
\mathop{\rm min}_\gv \; \| \gv_L - \fv_L \|^2 \;\; \text{subject to} \;\; \omega(\gv) \leq \theta,
\label{eq:bl_interp}
\end{equation}
where $\gv_L$ and $\fv_L$ denote the values of $\gv$ and $\fv$, respectively, on the set $L$. The constraint restricts the hypothesis space to a set of bandlimited signals with bandwidth less than $\theta$, which is equivalent to enforcing smoothness of the labels over the graph. This method essentially improves upon the Fourier eigenvector approach suggested in~\cite{belkin_nips_02,belkin_ml_04} in two ways: first, label prediction can be carried out without explicitly computing the eigenvectors of $\Lm$ using efficient iterative approaches implemented via graph filtering operations~\cite{narang_globalsip13,wang_tsp_15}. And second, one can also use the sampling theorem for graph signals to set $\theta$ as the \emph{cutoff frequency} $\omega_c(L)$ associated with the labeled set~\cite{anis_icassp_14,anis_tsp_16}, which, for a given $L$, is defined as the bandwidth below which any bandlimited signal is uniquely represented by its values on $L$. This approach is taken in~\cite{narang_icassp13,gadde_kdd_14}, and is particularly useful when $\omega(\fv) < \omega_c(L)$, in which case the minimizer $\gv^*$ of~\eqref{eq:bl_interp} exactly equals $\fv$, i.e., $\|\gv^* - \fv\| = 0$. Alternatively, one can also reconstruct $\fv$ using the variational problem: $\mathop{\rm min}_\gv \; \omega(\gv) \;\; \text{subject to} \;\; \gv_L = \fv_L$; the minimizer in this case is also exactly equal to $\fv$ if $\omega(\fv) < \omega_c(L)$~\cite{pesenson_ca_09,anis_tsp_16}. Further, it also possible to provide error bounds for both methods when $\omega(\fv) < \omega_c(L)$ is not satisfied~\cite{anis_tsp_16}.

The bandwidth of any indicator signal is also useful in specifying the amount of labeling required for its recovery in the context of sampling theory, as demonstrated by the following key result~\cite{anis_tsp_16}: 
\begin{lemma}
\label{lemma:num_labels}
Let $\mathcal{N}_\Lm(t)$ denote the number of eigenvalues of $\Lm$ less than or equal to $t$. Then, for any signal $\fv$ with bandwidth $\omega(\fv)$, there exists a subset of nodes $T \subseteq V$ of size $|T| = \mathcal{N}_\Lm(\omega(\fv))$ such that $\fv$ can be perfectly recovered from its values $\fv_T$ on $T$.
\end{lemma}
\begin{proof}
Since $\fv$ has bandwidth $\omega(\fv)$, it is spanned by the first $\mathcal{N}_\Lm(\omega(\fv))$ eigenvectors of $\Lm$, i.e., let $R := \{1,\dots,r\}$, then we have
\begin{equation}
\fv = \sum_{i=1}^{\mathcal{N}_\Lm(\omega(\fv))} c_i \uv_i = \Um_{:,R}\cv,
\label{eq:gft}
\end{equation}
where $c_i \neq 0$ for $i = \mathcal{N}_\Lm(\omega(\fv))$ and $\Um_{:,R}$ denotes the rectangular matrix formed using the first $r$ eigenvectors of $\Lm$. Since the eigenvectors $\{\uv_i\}$ are orthogonal, $\Um_{:,R}$ has rank $r = \mathcal{N}_\Lm(\omega(\fv))$. Therefore, there exists a subset of rows, indexed by a set $T$, with cardinality $|T| = r = \mathcal{N}_\Lm(\omega(\fv))$, such that the $r \times r$ matrix $\Um_{T,R}$ is full-rank, and thus invertible. Using this in \eqref{eq:gft}, we get $\cv = \Um_{T,R}^{-1} \fv_T$ and thus $\fv$ can be perfectly recovered from $\fv_T$ as $\fv = \Um_{V,R} \Um_{T,R}^{-1} \fv_T$, thereby proving our claim. Note that this is exactly the closed-form solution of~\eqref{eq:bl_interp}, for $L = T$ and $\theta = \omega(\fv)$, when the eigenvectors of $\Lm$ are known.
\end{proof}
We shall use this result later, to compute the label complexity of any signal $\fv$ on the graph as $\frac{1}{n} \Nc_\Lm\left(\omega(\fv)\right)$. Note, however, that this quantity only specifies the fraction of nodes to label on the graph -- selecting which nodes to label is another question altogether. This problem has been well-studied as part of graph sampling theory~\cite{anis_icassp_14,anis_tsp_16,shomorony_globalsip_14,chen_tsp_15}, with consideration of other important issues such as stability of reconstruction and computational complexity.

\subsection{Estimating bandwidth for graph signals}

Ideally, computing the bandwidth $\omega(\fv)$ of a graph signal $\fv$ requires obtaining the eigenvectors $\{\uv_i\}$ of $\Lm$ and the corresponding projections $\tilde{\fv}_i = \uv_i^T \fv$. However, analyzing the convergence of these coefficients is technically challenging. Therefore, we resort to the following estimate of the bandwidth~\cite{anis_tsp_16}:
\begin{equation}
\omega_m(\fv) := \left( \frac{\fv^T \Lm^m \fv}{\fv^T \fv} \right)^{1/m},
\end{equation}
where we call $\omega_m(\fv)$ the $m^\text{th}$-order bandwidth estimate.
It can be shown that the bandwidth estimates satisfy the property: for all $ 0 < m_1 < m_2$, $\omega_{m_1}(\fv) \leq \omega_{m_2}(\fv) \leq \omega(\fv)$. In other words, $\{\omega_m(\fv)\}$ forms a monotonically improving sequence of estimates of the true bandwidth $\omega(\fv)$. Further, we can also show~\cite{anis_tsp_16}:
\begin{equation}
\forall \fv, \; \omega(\fv) = \lim_{m \rightarrow \infty} \omega_m(\fv).
\label{eq:bw}
\end{equation} 

\subsection{Focus of this paper}

The discussion in Section~\ref{subsec:sampling_theory} indicates that in the discrete setting, with finite number of data points, the notions of bandwidth, bandlimited reconstruction and label complexity are well-motivated and quite useful in highlighting a sampling theory perspective of graph-based semi-supervised learning. However, there is a lack of understanding of these concepts in terms of their geometrical interpretation, i.e., their connection with the underlying geometry of the dataset. Thus, inspired by existing analysis in the literature for popular graph-based smoothness measures, we seek to bridge this gap by analyzing these concepts in the asymptotic regime of infinite data points for the data models and graph construction scheme described earlier.

Analyzing the convergence of the bandwidth estimates of class indicator signals for the separable and the nonseparable models constitutes the main subject for the rest of this paper. Our approach, similar to existing results in the literature, starts in the discrete domain by drawing $n$ samples from the data models, constructs a sequence of graphs $G_{n,\sigma}$ from the data points, and considers the behavior of
\begin{equation*}
\omega_m(\onev_S) = \left( \frac{\onev_S^T \Lm^m \onev_S}{\onev_S^T \onev_S} \right)^{\frac{1}{m}} \text{and}\;\; \omega_m(\onev_A) = \left( \frac{\onev_A^T \Lm^m \onev_A}{\onev_A^T \onev_A} \right)^{\frac{1}{m}}
\end{equation*}
over the graphs as $n \rightarrow \infty$, $\sigma \rightarrow 0$ and $m \rightarrow \infty$. Intuitively, the condition $n \rightarrow \infty$ implies an abundance of unlabeled data, $\sigma \rightarrow 0$ dictates that the connectivity around each node is meaningful and does not blow up, and $m \rightarrow \infty$ translates to improving estimates of the bandwidth. Our analysis relates $\omega_m(\onev_S)$ and $\omega_m(\onev_A)$ to the underlying data distribution $p(\xv)$ and class boundaries -- the hypersurface $\partial S$ in the separable case and the overlap region $\partial A$ in the nonseparable case. Using these results, we also comment on the label complexities of reconstructing $\onev_S$ and $\onev_A$ over the graph in the asymptotic limit.


\section{Related work and connections}
\label{sec:related_work}
\begin{table*}[t]
\centering
\caption{Related convergence results in the literature under different data models and graph construction schemes. All models assume that the distributions are smooth (at least twice-differentiable).
Further, the graph Laplacian is defined as $\Lm = \frac{1}{n} (\Dm - \Wm)$ in all cases.
\cite{maier_13} also studies convergence of graph cuts for weighted $k$-nearest neighbor and $r$-neighborhood graphs which we do not include for brevity.}
\label{tab:related_work}
\begin{tabular}{p{0.5in}|p{1.6in}|p{1.6in}|p{0.7in}|p{0.8in}|p{0.9in}}
\hline
\hline
\emph{Work} &
\emph{Data model} &
\emph{Graph model} &
\emph{Quantity} &
\emph{Convergence regime} &
\emph{Limit (within constant scaling factor)} 
\\
\hline
\hline
Narayanan \emph{et al}~\cite{narayanan_06} &
$p(\xv)$ supported on manifold $\Mc \subset \mathbb{R}^d$, separated into $S$ and $S^c$ by smooth hypersurface $\partial S$ &
Normalized Gaussian weights $w'_{ij} = \frac{w_{ij}}{\sqrt{d_i d_j}}$ &
$ \frac{1}{n\sigma} \onev_S^T \Lm \onev_S $   & $n \to \infty$, $\sigma \to 0$ &
$\int_{\partial S} p(\sv) d\sv$
\\
\hline
Maier \emph{et al} \cite{maier_13} &
$p(\xv)$ supported on $\Mc \subset \mathbb{R}^d$, separated into $S$ and $S^c$ by hyperplane $\partial S$ &
$r$-neighborhood, unweighted &
$ \frac{1}{n r^{d+1}} \onev_S^T \Lm \onev_S $ & $n \to \infty$, $r \to 0$ &
$\int_{\partial S} p^2(\sv)d\sv$
\\
\cline{3-6} &
&
$k$-nn, unweighted, $t = (k/n)^{1/d}$ &
$ \frac{1}{n t^{d+1}} \onev_S^T \Lm \onev_S $ & $n \to \infty$, $t\to 0$ &
$\int_{\partial S} p^{1-1/d}(\sv) d\sv$ \\
\cline{3-6} &
&
fully-connected, Gaussian weights &
$ \frac{1}{n\sigma} \onev_S^T \Lm \onev_S$ &
$n \to \infty$, $\sigma \to 0$ &
$\int_{\partial S} p^2(\sv)d\sv$
\\
\hline
Bousquet \emph{et al} \cite{bousquet_04}, Hein \cite{hein_colt_06} &
$p(\xv)$ and $f(\xv)$ supported on $\mathbb{R}^d$ &
fully-connected, weights $w_{ij} =$ $\frac{1}{n\sigma^d} K\left(\frac{\| \Xm_i - \Xm_j\|^2}{\sigma^2}\right)$, where $K(.)$ is a smooth decaying kernel &
$ \frac{1}{n\sigma^2} \fv^T \Lm \fv$ &
$n \to \infty$, $\sigma \to 0$ &
$\int \| \nabla f(\xv)\|^2 p^2(\xv) d\xv$
\\
\hline
Zhou \emph{et al} \cite{zhou_aistats_11} &
Uniformly distributed on $d$-dim. submanifold $\Mc$ &
fully-connected, Gaussian weights &
$ \frac{1}{n\sigma^m} \fv^T \Lm^m \fv$ &
$n \to \infty$, $\sigma \to 0$ &
$\int f(\xv) \Delta^m  f(\xv) d\xv$
\\
\hline
Garc{\'i}a Trillos \& Slep{\v c}ev \cite{trillos_arma_16} &
p(\xv) supported on $D \subset \mathbb{R}^d$ &
fully-connected, weights $w_{ij} =$ $\frac{1}{\varepsilon^d}\eta\left(\frac{\| \Xm_i - \Xm_j\|}{\varepsilon}\right)$, where $\eta(.)$ is a smoothly decaying kernel  &
$ \frac{1}{n^2 \varepsilon} GTV(\fv) $ &
$n \to \infty$, $\varepsilon \to 0$ &
$\int \| \nabla f(\xv)\| p^2(\xv) d\xv$
\\
\hline
El Alaoui \emph{et al} \cite{elalaoui_colt_16}, Slep{\v{c}}ev \& Thorpe~\cite{slepcev_arxiv_17} &
p(\xv) supported on $[0,1]^d$, $\Omega \subset \mathbb{R}^d$ &
fully-connected, weights $w_{ij} =$ $\phi\left(\frac{\| \Xm_i - \Xm_j\|}{h}\right)$, where $\phi(.)$ is a smoothly decaying kernel &
$ \frac{1}{n^2 h^{p + d}} J_p(\fv) $ &
$n \to \infty$, $h \to 0$ &
$\int \| \nabla f(\xv)\|^p p^2(\xv) d\xv$
\\
\hline
\hline
This work &
$p(\xv)$ supported on $\mathbb{R}^d$, separated into $S$ and $S^c$ by smooth hypersurface $\partial S$ &
fully-connected, Gaussian weights &
$\frac{1}{\sigma^{1/m}} \omega_m(\onev_S)$ &
$n \to \infty$, $\sigma \to 0$, $m \to \infty$ & $\sup_{\sv \in \partial S} p(\sv)$
\\
\cline{2-6} &
Drawn from $p_A(\xv)$ and $p_{A^c}(\xv)$ supported on $\mathbb{R}^d$ with probabilities $\alpha_A$ and $\alpha_{A^c}$ &
fully-connected, Gaussian weights &
$ \omega_m(\onev_A)$ &
$n \to \infty$, $\sigma \to 0$, $m \to \infty$ &
$\sup_{\xv \in \partial A} p(\xv)$ 
\\     
\hline
\hline
\end{tabular}
\end{table*}%
%
Existing convergence analyses of the graph-based smoothness measures for various graph construction schemes appear in two different settings -- \emph{classification} and \emph{regression}.
The classification setting assumes that labels indicate class memberships and are discrete, typically with $1/0$ values. Note that both the separable and nonseparable data models considered in our paper are in the classification setting. On the other hand, in the regression setting, one allows the class label signal $\fv$ to be sampled from a smooth function on $\mathbb{R}^d$ with soft values, such that $\fv \in \mathbb{R}^n$, and later applies some thresholding mechanism to infer class memberships. For example, in the two class problem, one can assign $+1$ and $-1$ to the two classes and threshold $\fv$ at $0$.
Convergence analysis of smoothness measures in this setting requires different scaling conditions than the classification setting, and leads to fundamentally different limit values that require differentiability of the label functions in the continuum. Applying these to class indicator functions may lead to ill-defined results.
A summary of convergence results in the literature for both settings is presented in Table~\ref{tab:related_work}. Although these results do not focus on analyzing the bandwidth of class indicator signals, the proof techniques used in this paper are inspired by some of these works. We review them in this section and discuss their connections to our work.

\subsection{Classification setting}

Prior work under this setting assumes the separable data model where the feature space is partitioned by smooth decision boundaries into different classes. When $m=1$, the bandwidth estimate $\omega_m(\onev_S)$ for the separable model in our work reduces (within a scaling factor) to the empirical graph cut for the partitions $S$ and $S^c$ of the feature space, i.e.,
\begin{equation} 
\textit{Cut}(S,S^c) := \sum_{\Xm_i \in S, \Xm_j \in S^c} w_{ij} = n \onev_S^T \Lm \onev_S \;.
\end{equation}
Convergence of this quantity has been studied before in the context of spectral clustering, where one tries to minimize it across the two partitions of the nodes.
It has been shown in~\cite{maier_13} that the cut formed by a hyperplane $\partial S$ in $\mathbb{R}^d$ converges with some scaling under the rate conditions $\sigma \rightarrow 0$ and $n\sigma^{d+1} \rightarrow \infty$ as
\begin{equation}
\label{eq:result_maier}
\frac{1}{n\sigma} \onev_S^T \Lm \onev_S \xrightarrow{p.} \frac{1}{\sqrt{2\pi}}\int_{\partial S} p^2(\sv)d\sv,
\end{equation}
where $d\sv$ ranges over all $(d-1)$-dimensional volume elements tangent to the hyperplane $\partial S$, and $p.$ denotes convergence in probability.
The analysis has also been extended to other graph construction schemes such as the $k$-nearest neighbor graph and the $r$-neighborhood graph, both weighted and unweighted. The condition $\sigma \rightarrow 0$ in~\eqref{eq:result_maier}
is required to have a clear and well-defined limit on the right hand side. We borrow this convergence regime in our work, since it allows a succinct interpretation of the bandwidth of class indicator signals. Intuitively, it enforces sparsity in the similarity matrix $\Wm$ by shrinking the neighborhood volume as the number of data points increases. As a result, one can ensure that the graph remains sparse even as the number of points goes to infinity. 
A similar result for a similarity graph constructed with normalized weights $w'_{ij} = w_{ij}/\sqrt{d_i d_j}$ was shown earlier for an arbitrary hypersurface $\partial S$ in \cite{narayanan_06}, where $d_i$ denotes the degree of node $i$.
In this case, normalization of the graph weights results in convergence to $\frac{1}{\sqrt{2\pi}} \int_{\partial S} p(\sv)d\sv$. Similarly, in~\cite{castro_aap_12}, the convergence of normalized cuts is analyzed for points drawn from a uniform density.
All of these results aim to provide an interpretation for spectral clustering -- up to some scaling, the empirical cut value converges to a weighted volume of the boundary. Thus, spectral clustering is a means of performing low density separation on a finite sample drawn from a distribution in feature space.

Note that these works provide little insight for the convergence analysis of higher-order regularizers, i.e., $\omega_m(\onev_S)$ for $m > 1$ in our case, since these require different scaling factors and rate conditions. Further, we get no clue about the continuum limit values of $\omega_m(\onev_S)$ and $\omega_m(\onev_A)$ from any of these results. However, the definition and some of the proof techniques we use for the separable models in this paper have been inspired by~\cite{narayanan_06,maier_13}.

\subsection{Regression setting}

To predict the labels of unknown samples in the regression setting, one generally minimizes the graph Laplacian regularizer $\fv^T \Lm \fv$ subject to the known label constraints \cite{zhu_03}:
\begin{equation}\label{eq:ssl_formulation}
\mathop{\rm min}_\fv \; \fv^T \Lm \fv \text{ such that } \fv(L) = \yv(L),
\end{equation}
One particular convergence result in this setting assumes that $n$ data points are drawn \emph{i.i.d.} from $p(\xv)$ and are labeled by sampling a smooth function $f(\xv)$ on $\mathbb{R}^d$. Here, the graph Laplacian regularizer $\fv^T \Lm \fv$ can be shown to converge in the asymptotic limit under the conditions $\sigma \rightarrow 0$ and $n\sigma^{d} \rightarrow \infty$ as in ~\cite{bousquet_04, hein_colt_06}:
\begin{equation}\label{eq:beliefresult}
\frac{1}{n \sigma^2} \fv^T \Lm \fv \xrightarrow{\;\;p.\;} C \int_{\mathbb{R}^d} \| \nabla f(\xv) \|^2 p^2(\xv) d\xv,
\end{equation}
where for each $n$, $\fv$ is the $n$-dimensional label vector representing the values of $f(\xv)$ at the $n$ sample points, $\nabla$ is the gradient operator and $C$ is a constant factor independent of $n$ and $\sigma$. The right hand side of the result above is a weighted Dirichlet energy functional that penalizes variation in the label function weighted by the data distribution. Similar to the justification of spectral clustering, this result justifies using the formulation in \eqref{eq:ssl_formulation} for semi-supervised classification: given label constraints, the predicted label function must vary little in regions of high density.
The work of~\cite{hein_colt_06,hein_thesis_06} generalizes this result by using arbitrary kernel functions for defining graph weights, and defining data distributions over manifolds in $\mathbb{R}^d$.
Convergence results for another regularizer called Graph Total Variation, defined as $GTV(\fv) = \sum_{i,j}w_{ij}|f_i - f_j|$, are presented in~\cite{trillos_arma_16,trillos_arxiv_16}. For data points drawn from $p(\xv)$ defined over a domain $D \subset \mathbb{R}^d$, graph weights given by $w_{ij} = \frac{1}{\varepsilon^d} \eta\left( \frac{\|\Xm_i-\Xm_j\|}{\varepsilon} \right)$, one has as $n\rightarrow \infty$ and $\varepsilon \rightarrow 0$: 
\begin{equation}
\frac{1}{n^2\varepsilon} GTV(\fv) \xrightarrow{\;\;\Gamma\;} C\int_D \|\nabla f(\xv)\|p^2(\xv)d\xv,
\end{equation}
where the limit is analyzed in the setting of $\Gamma$-convergence~\cite{trillos_arma_16}.
These results extend to the classification setting when $f(\xv)$ is an indicator function, for example, the limit for $f(\xv) = 1_S(\xv)$ reduces to that of~\eqref{eq:result_maier}. This approach is used in~\cite{trillos_jmlr_16} to analyze convergence of Cheeger and ratio cuts.

Similar convergence results have also been derived for the higher-order Laplacian regularizer $\fv^T \Lm^m \fv$ obtained from uniformly distributed data~\cite{zhou_aistats_11}. In this case, it was shown that for data points obtained from a uniform distribution on a $d$-dimensional submanifold $\Mc \subset \mathbb{R}^N$ such that ${\rm Vol} (\Mc) = 1$ and $2m$-differentiable functions $f(\xv)$, one has as $n \rightarrow \infty$:
\begin{equation}
\label{eq:higher_order}
\frac{1}{n\sigma_n^m} \fv^T \Lm^m \fv \xrightarrow{\;\;p.\;} C \int_{\Mc} f(\xv) \Delta^m f(\xv) d\xv,
\end{equation}
where $\Delta$ is the Laplace operator and $\sigma_n = n^{-1/(2d + 4 + \alpha)}$ is a vanishing sequence with $\alpha > 0$. Extensions for non-uniform probability distributions $p(\xv)$ over the manifold can be obtained using the weighted Laplace-Beltrami operator~\cite{belkin_jcss_08,zhou_kdd_11}. More recently, an $\ell_p$-based Laplacian regularization has been proposed for imposing smoothness constraints in semi-supervised learning problems~\cite{buhler_icml_09,elalaoui_colt_16}. This is similar to a higher-order regularizer but is defined as $J_p(\fv) := \sum_{i,j \in E} w_{ij}^p |f_i - f_j|^p$, where $w_{ij} = \phi(\|\Xm_i - \Xm_j\| / h)$ and $\phi(.)$ is a smoothly decaying Kernel function. It has been shown for a bounded density $p(\xv)$ defined on $[0,1]^d$ that for every $p \geq 2$, as $n \rightarrow \infty$, followed by $h \rightarrow 0$,
\begin{equation}
\label{eq:p_laplacian}
\frac{1}{n^2 h^{p + d}} J_p(\fv) \xrightarrow{\;\;p.\;} C \int_{[0,1]^d} \|\nabla f(\xv)\|^p p^2(\xv) d\xv.
\end{equation}
The work of~\cite{slepcev_arxiv_17} generalizes this result over an open, bounded and connected set $\Omega \subset \mathbb{R}^d$ and analyzes rate conditions such that the scalings $n\rightarrow \infty$, $h\rightarrow 0$ occur jointly.  

Note that although our work also uses higher powers of $\Lm$ in the expressions for $\omega_m(\onev_S)$ and $\omega_m(\onev_A)$, we cannot use the  convergence results in~\eqref{eq:higher_order} and the proof techniques of~\eqref{eq:p_laplacian}, since they are only applicable for smooth functions (i.e., differentiable up to a certain order) on $\mathbb{R}^d$. Specifically, in our case, $\onev_S$ in the separable model is sampled from a discontinuous indicator function $1_S(\xv)$, hence plugging it into existing results does not give a meaningful result for higher values of $m$. Further, the nonseparable model can only be defined in the classification setting, i.e., $\onev_A$ in the nonseparable model does not have a continuum counterpart. Therefore, our analysis has to take a different route that has more similarities with the proof techniques used for the classification setting. We shall later see that a bulk of the effort in proving our results goes into expanding $\omega_m(\onev_S)$ and $\omega_m(\onev_A)$ for any $m$ by keeping track of every term in the expansion. This is followed by a careful evaluation of the integrals in their expected values by reducing them term-by-term.


%
%
%
%
\section{Main results and Discussion}
\label{sec:main_results}
\subsection{Interpretation of bandwidth and bandlimited reconstruction}

We first show that under certain conditions, the bandwidth estimates of class indicator signals for both the data models, i.e., $\omega_m(\onev_S)$ and $\omega_m(\onev_A)$, over Gaussian kernel-based similarity graphs $G_{n,\sigma}$ constructed from data points in $\Xc$, converge to quantities that are functions of the underlying distribution and the class boundary for both data models.
This convergence is achieved under the following asymptotic regime:
\begin{enumerate}
\item Increasing size of dataset: $n \rightarrow \infty$.
\item Shrinking neighborhood volume: $\sigma \rightarrow 0$.
\item Improving bandwidth estimates: $m \rightarrow \infty$.
\end{enumerate}
Note that an increasing size of the dataset $n \rightarrow \infty$ is required for the stochastic convergence of the bandwidth estimate.
$\sigma \rightarrow 0$ ensures that the limiting values are concise and have a simple interpretation in terms of the data geometry.
Intuitively, as the number of data points increases, the neighborhood around each data point shrinks -- as a result, the degree of each node in the graph does not blow up. Finally, $m \rightarrow \infty$ leads to improving values of the bandwidth estimate.

The convergence results are precisely stated in the following theorems:
\begin{theorem}
If $n \rightarrow \infty$, $\sigma \rightarrow 0$ and $m \rightarrow \infty$ while satisfying the following rate conditions
\begin{enumerate}
\item $(n \sigma^{md+1})/(m^2C^m) \rightarrow \infty$, where $C = 2/(2\pi)^{d/2}$,
\item $m 2^m \sigma \rightarrow 0$,
\end{enumerate}
then for the separable model, one has
\begin{equation}
\frac{1}{\sigma^{1/m}} \omega_m(\onev_S) \xrightarrow[]{\;\; \textit{p.} \;} \sup_{\sv \in \partial S} p(\sv),
\label{eq:bwresult_sep}
\end{equation}
where ``p." denotes convergence in probability.
\label{thm:conv_sep}
\end{theorem}
\begin{theorem}
If $n \rightarrow \infty$, $\sigma \rightarrow 0$ and $m \rightarrow \infty$ while satisfying the following rate conditions
\begin{enumerate}
\item $(n \sigma^{md})/(m^2C^m) \rightarrow \infty$, where $C = 2/(2\pi)^{d/2}$,
\item $m2^m\sigma^2 \rightarrow 0$,
\end{enumerate}
then for the non-separable model, one has
\begin{equation}
\omega_m(\onev_A) \xrightarrow[]{\;\; \textit{p.} \;} \sup_{\xv \in \partial A} p(\xv).
\label{eq:bwresult_nonsep}
\end{equation}
\label{thm:conv_nonsep}
\end{theorem}
The dependence of the results on the rate conditions will be explained later in the proofs section. An example of parameter choices for scaling laws to hold simultaneously is illustrated in the following remark:
\begin{remark}
Equations~\eqref{eq:bwresult_sep} and~\eqref{eq:bwresult_nonsep} hold if for each value of $n$, we choose $m$ and $\sigma$ as follows:
\begin{align}
m &= [m_0 \; (\log{n})^y], \\
\sigma &= \sigma_0 \; n^{-x/md },
\end{align}
for constants $m_0, \sigma_0 > 0$, $0 < y < 1/2$ and $0 < x < 1$. $[\;.\;]$ indicates taking the nearest integer value.
\label{cor:bw_ex} 
\end{remark}

Theorems~\ref{thm:conv_sep} and~\ref{thm:conv_nonsep} give an explicit connection between bandwidth estimates of class indicator signals and class boundaries in the dataset. This interpretation forms the basis of justifying the choice of bandwidth as a smoothness constraint in graph-based learning algorithms.
Theorem~\ref{thm:conv_sep} suggests that for the separable model, if the boundary $\partial S$ passes through regions of low probability density, then the bandwidth of the corresponding class indicator vector $\omega(\onev_S)$ is low. A similar conclusion is suggested for the nonseparable model from Theorem~\ref{thm:conv_nonsep}, i.e., if the density of data points in the overlap region $\partial A$ is low, then the bandwidth $\omega(\onev_A)$ is low. In other words, low density of data in the boundary regions leads to smooth indicator functions.

From our results, we also get an intuition behind the smoothness constraint imposed in the bandlimited reconstruction approach~\eqref{eq:bl_interp} for semi-supervised learning. Basically, enforcing smoothness on classes in terms of indicator bandwidth ensures that the algorithm chooses a boundary passing through regions of low data density in the separable case. Similarly, in the nonseparable case, it ensures that variations in labels occur in regions of low density. Further, the bandwidth constraint $\theta$ in~\eqref{eq:bl_interp} effectively imposes a constraint on the complexity of the hypothesis space -- a larger value increases the size of the hypothesis space and opens up choices consisting of more complex boundaries.

%

%
Note that Theorems~\ref{thm:conv_sep} and~\ref{thm:conv_nonsep} can be improved and their assumptions generalized in several ways:
\begin{itemize}[leftmargin=10pt]
\item The convergence results can be generalized to graphs with edge weights computed using any non-increasing kernel $\eta_\sigma(\|\zv\|) = \frac{1}{\sigma^d} \eta(\|\zv\|)$, where $\sigma$ is a scaling parameter that controls the kernel width and goes to zero as $n\rightarrow \infty$. The limits of $\omega_m(\onev_S)$ and $\omega_m(\onev_A)$ stay the same as in~\eqref{eq:bwresult_sep} and~\eqref{eq:bwresult_nonsep}, up to a constant factor.
\item The domain of the data density $p(\xv)$ can be generalized to open, bounded and connected sets $D \subset \mathbb{R}^d$ with Lipschitz boundary similar to the work of~\cite{trillos_arma_16,trillos_jmlr_16,trillos_arxiv_16,slepcev_arxiv_17}, or a low dimensional compact manifold embedded in $\mathbb{R}^d$ as in~\cite{narayanan_06,hein_colt_06}.
\item Convergence of the bandwidth estimates $\omega_m(\onev_S)$ and $\omega_m(\onev_A)$ does not imply convergence of the actual bandwidths $\omega(\onev_S)$ and $\omega(\onev_A)$, respectively, to the same continuum limiting values. This is because the scaling of $m$ is tied to $n$ and $\sigma$ in our rate conditions, whereas ideally, one should take the limit $m \rightarrow \infty$ first, and independently of $n$ and $\sigma$ while analyzing the estimates. In this case, the scaling factor $\frac{1}{\sigma^{1/m}}$ in the left hand side of~\eqref{eq:bwresult_sep} also disappears. The analysis for this interchange of limits is challenging and we do not know how to approach this problem at the moment, so we leave it for future work. However, based on experiments in Section~\ref{sec:exp}, where we use actual bandwidths instead of their estimates to validate convergence, we conjecture that the same results hold for both, i.e.,
\begin{conjecture}
\label{conjecture:bw}
As $n \rightarrow \infty$ and $\sigma \rightarrow 0$ at appropriate rates, $\omega(\onev_S) \rightarrow \sup_{\sv \in \partial S} p(\sv)
$ and $\omega(\onev_A) \rightarrow \sup_{\xv \in \partial A}$.
\end{conjecture}
\item Note that Theorems~\ref{thm:conv_sep} and~\ref{thm:conv_nonsep} show pointwise convergence for fixed underlying data models, i.e., convergence is proven for a given indicator signal $\onev_S$ specified by $\{ p(\xv), \partial S \}$, and $\onev_A$ specified by $\{p_A(\xv), p_{A^c}(\xv)\}$. This is not sufficient when we want to interpret the behavior of a bandwidth-based learning algorithm, since we cannot guarantee that the solution returned by the algorithm matches the solution of its continuum limit version. We need stronger convergence results for this case, such as those recently covered in~\cite{trillos_arma_16,trillos_jmlr_16,slepcev_arxiv_17}.
\end{itemize}

Finally, as a special case of our analysis, we also get a convergence result for the graph cut in the nonseparable model analogous to the results of~\cite{maier_13} for the separable model. Note that the cut in this case equals the sum of weights of edges connecting points that belong to class $A$ to points that do not belong to class $A$, i.e.,
\begin{equation}
\textit{Cut}(A,A^c)  := \sum_{\Xm_i \in A, \Xm_j \in A^c} w_{ij} = n\onev_A^t \Lm \onev_A.
\end{equation}
With this definition, we have the following result:
\begin{theorem}
\label{thm:cut_nonsep}
If $n \rightarrow \infty$, $\sigma \rightarrow 0$ such that $n\sigma^{d} \rightarrow \infty$, then
\begin{equation}
\frac{1}{n} \text{Cut}(A,A^c) \xrightarrow[]{\;\; \textit{p.} \;} \int \alpha_A \alpha_{A^c} p_A(\xv) p_{A^c}(\xv) d\xv.
\end{equation}
\end{theorem}
The result above indicates that if the overlap between the conditional distributions of a particular class and its compliment is low, then the value of the graph cut is lower. This justifies the use of spectral clustering in the context of nonseparable models.

\subsection{Label complexity}
\begin{figure*}
\begin{center}
\begin{subfigure}{0.45\textwidth}
\centering
\includegraphics[width=0.85\linewidth]{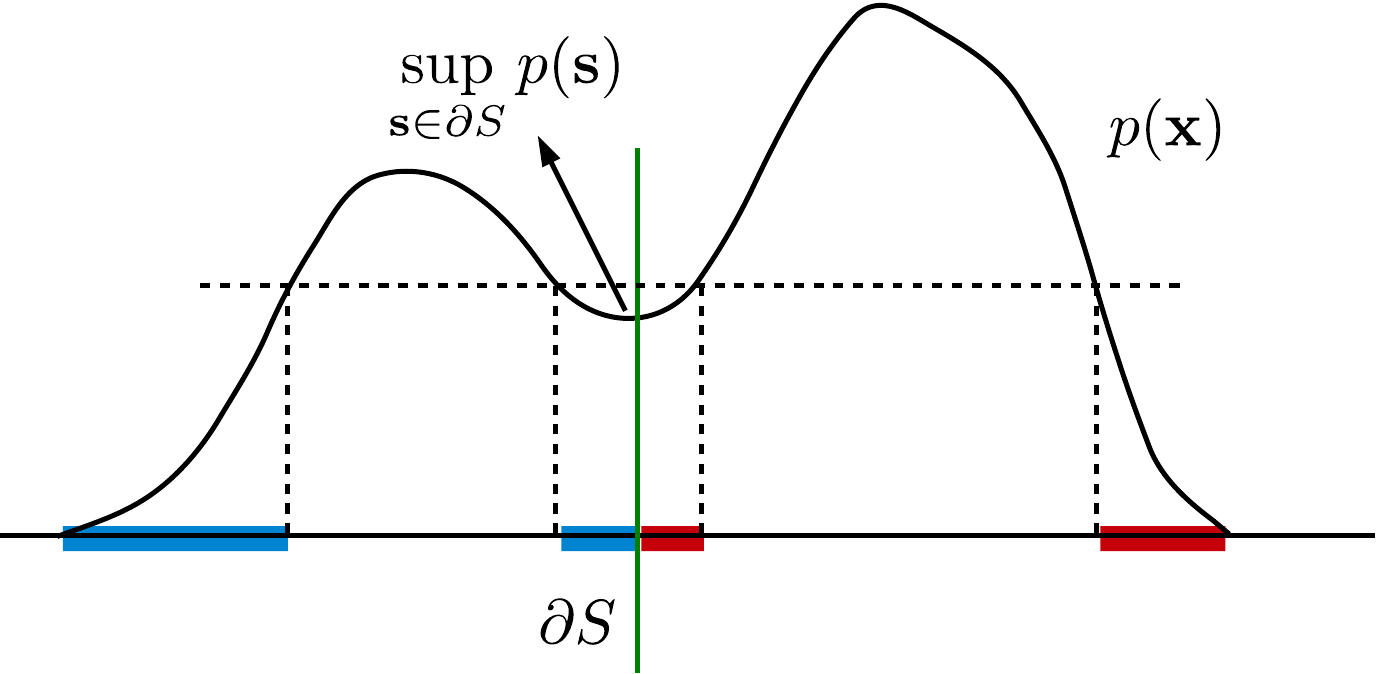}
\caption{}
\end{subfigure}
\qquad
\begin{subfigure}{0.45\textwidth}
\centering
\includegraphics[width=0.85\linewidth]{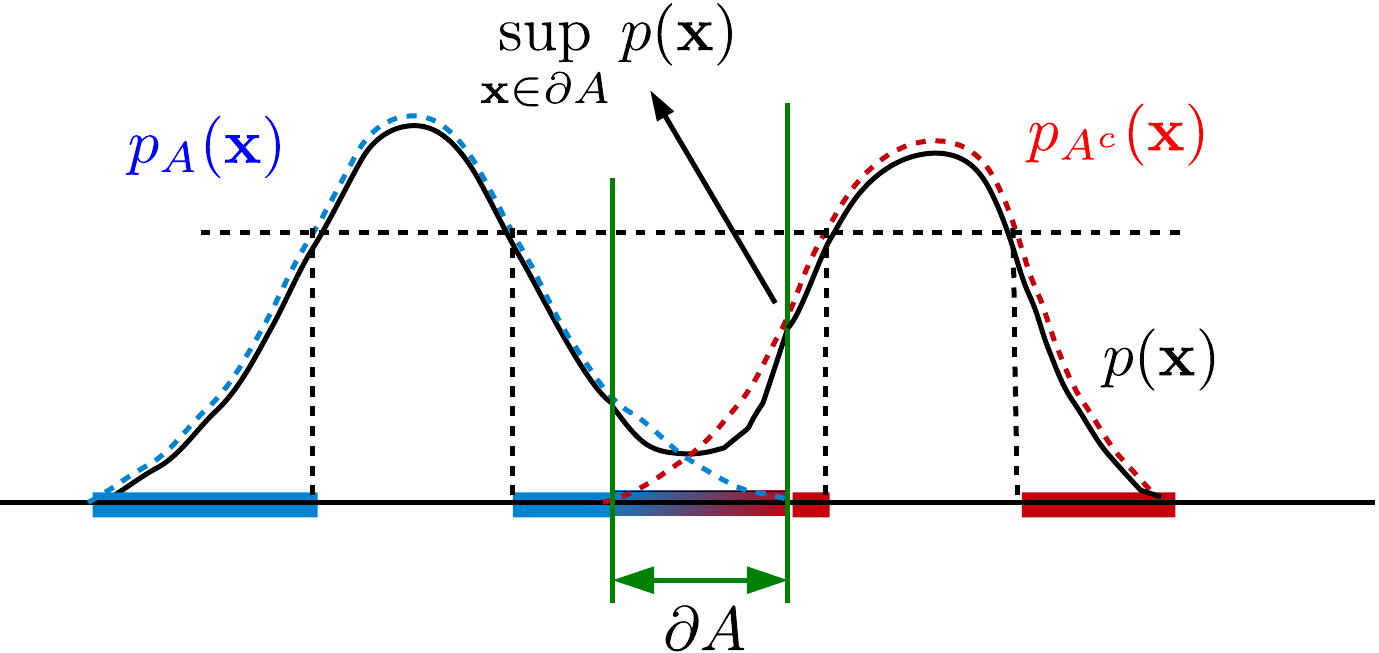}
\caption{}
\end{subfigure}
\end{center}
\caption{1-D example illustrating ideal label complexities for (a) the separable model, and (b) the nonseparable model. Note that for an unlabeled dataset, labeling all points for the sublevel set defined by the supremum density over the boundary resolves all ambiguity and results in perfect prediction of the unknown labels.}
\label{fig:lab_comp}
\end{figure*}
In the context of our work, we define the label complexity of learning class indicators over the graph using a sampling theoretic approach, as the fraction of labeled nodes required for perfectly predicting the labels of the unlabeled nodes. Formally, for a given class indicator $\onev_C \in \{ 0,1,\}^n$ over the graph $G_n$, we define it as the fraction of points that need to be labeled so that a sampling theory-based reconstruction algorithm (such as bandlimited reconstruction of~\eqref{eq:bl_interp}) outputs a solution $\fv^*$ with zero reconstruction error: $\| \fv^* - \onev_C \| = 0$. Note that perfect reconstruction is a strong requirement that can be relaxed by allowing an error tolerance $\epsilon$, in which case the amount of labeling required is lower. However, this requirement simplifies our analysis since we can directly use results from sampling theory to evaluate this quantity. Specifically, we can simply use Lemma~\ref{lemma:num_labels} to calculate the label complexity for $\onev_C$ over the graph as $\frac{1}{n} \Nc_\Lm(\omega(\onev_C))$. In our context, label complexity is essentially an indicator of how ``good" the semi-supervised problem is, i.e., how much help we get from geometry while predicting the unknown labels. A low label complexity is indicative of a favorable situation, where one is able to learn from only a few known labels by exploiting data geometry.

Note that our definition of label complexity is concerned with reconstructing class indicators only on the nodes of the graph. This pertains to the transductive learning philosophy, a common setting considered in most graph-based semi-supervised learning literature, where the goal is to simply predict the labels of the unlabeled points and not learn a general labeling rule/classifier. Further, our definition is different and simpler than the more general $(\epsilon,\delta)$ definition of sample/label complexity in Probably Approximately Correct (PAC) learning~\cite{hanneke_jmlr_15}, i.e., it is concerned with reconstructing only a given class indicator, with zero error, using a sampling theory-based learning approach, over a graph constructed from a given data model.

\subsubsection*{Ideal label complexities}
A simple way to compute the label complexity, for the data models we consider, is to find the fraction of points belonging to a region that fully encompasses the boundary. To formalize this, 
let us define the following two sublevel sets in $\mathbb{R}^d$:
\begin{align}
\Xc_S &:= \{\xv : p(\xv) \leq \sup_{\sv \in \partial S} p(\sv)\},\\
\Xc_A &:= \{\xv : p(\xv) \leq \sup_{\xv \in \partial A} p(\xv)\}.
\end{align}
Note that by definition, $\partial S$ is fully contained in $\Xc_S$ and $\partial A$ is fully contained in $\Xc_A$ (see Figure~\ref{fig:lab_comp} for an example in $\mathbb{R}^1$). Therefore, to perfectly reconstruct the indicator signals $\onev_S$ and $\onev_A$ for any $n$, it is sufficient to know the labels of all points in $\Xc_S$ and $\Xc_A$, respectively, as this strategy removes all ambiguity in labeling the two classes; a good learning algorithm can simply propagate the known labels on to the unlabeled points. Based on this and using the law of large numbers, we arrive at the following conclusion:
\begin{remark}
\label{remark:label_complexity}
The ideal label complexities of learning $\onev_S$ and $\onev_A$ in the asymptotic limit are given by $P(\Xc_S)$ and $P(\Xc_A)$, respectively, where $P(\Omega) = \int_\Omega p(\xv)d\xv$.
\end{remark}

\subsubsection*{Label complexity of $\onev_S$ and $\onev_A$ using a sampling theory-based approach}
Note that from Lemma~\ref{lemma:num_labels}, we know that the label complexities  for $\onev_S$ and $\onev_A$ are given as $\frac{1}{n} \Nc_\Lm(\omega(\onev_S))$ and $\frac{1}{n} \Nc_\Lm(\omega(\onev_A))$, respectively.
Since our bandwidth convergence results relate the bandwidth of indicators for the two data models with data geometry, we only need to asymptotically relate the fraction of eigenvalues of $\Lm$ below any constant. This is achieved by first proving the following:
\begin{theorem}
\label{thm:esd}
Let $\Nc_\Lm(t)$ be the number of eigenvalues of $\Lm$ below a constant $t$. Then, as $n \rightarrow \infty$ and $\sigma \rightarrow 0$, we have
\begin{equation}
\E{\frac{1}{n}\mathcal{N}_\Lm(t)} \longrightarrow P\left( \{\xv : p(\xv) \leq t\} \right).
\end{equation}
\label{eq:esd}
\end{theorem}
\begin{proof}
See Section~\ref{sec:proof_esd}.
\end{proof}
Note that Theorem \ref{thm:esd} can be strengthened by proving convergence of $\frac{1}{n}\Nc_\Lm(t)$ rather than its expected value. This requires further analysis, which we leave for future work. Plugging in $\omega(\onev_S)$ and $\omega(\onev_A)$ in place of $t$ in Theorem~\ref{thm:esd}, and using the convergence results from Theorems~\ref{thm:conv_sep} and~\ref{thm:conv_nonsep}, and Conjecture~\ref{conjecture:bw}, we speculate the following convergence for the label complexities of $\onev_S$ and $\onev_A$:
\begin{conjecture}
\label{conjecture:lc}
As $n \rightarrow \infty$, $\sigma \rightarrow 0$, we have
\begin{align}
\label{eq:sep_lc_limit}
\frac{1}{n} \Nc_\Lm\left(\omega(\onev_S)\right) \rightarrow P(\Xc_S), \\
\label{eq:nonsep_lc_limit}
\frac{1}{n} \Nc_\Lm\left(\omega(\onev_A)\right) \rightarrow P(\Xc_A).
\end{align}
\end{conjecture}
The limiting values in~\eqref{eq:sep_lc_limit} and~\eqref{eq:nonsep_lc_limit} are the same as those predicted by Remark~\ref{remark:label_complexity}; this is encouraging as far as the validity of Conjecture~\ref{conjecture:lc} is concerned. Additionally, we see strong evidence in our experiments to support our claims; specifically, the average error of predicting the labels of the unlabeled nodes goes to zero as the fraction of labeled examples crosses the limit values of~\eqref{eq:sep_lc_limit} and~\eqref{eq:nonsep_lc_limit} (see Figure~\ref{fig:exp3}).

The limiting values in~\eqref{eq:sep_lc_limit} and~\eqref{eq:nonsep_lc_limit} essentially indicate how the low density separation assumption can benefit semi-supervised learning, since in this case, one can forgo the task of labeling a significant fraction of the points and still reconstruct the indicator by exploiting data geometry. A classic example of where this can be useful is the two-step learning process, where the first step uses semi-supervised learning in a transductive setting to create a large training set using a combination of unlabeled and labeled data, and the second step involves learning a classifier using supervised learning. If the low density separation is satisfied by the data, then semi-supervised learning using a sampling theory-based approach effectively reduces the sample complexity of the supervised learning step by a constant fraction, equal to the limiting values in~\eqref{eq:sep_lc_limit} and~\eqref{eq:nonsep_lc_limit}.

\section{Proofs}
\label{sec:proofs}

We now present the proofs\footnote{A partial sketch of the proof for the separable model is also provided in our parallel work~\cite{anis_15}; here we provide the complete proof.} of Theorems~\ref{thm:conv_sep} and~\ref{thm:conv_nonsep}. The main idea is to perform a variance-bias decomposition of the bandwidth estimate and then prove the convergence of each term independently. Specifically, for any indicator vector $\onev_R \in \{0,1 \}^n$, we consider the random variable:
\begin{align}
\left(\omega_m(\onev_R)\right)^{m} = \frac{\onev_R^T \Lm^m \onev_R}{\onev_R^T \onev_R} = \frac{ \frac{1}{n} \onev_R^T \Lm^m \onev_R }{ \frac{1}{n} \onev_R^T \onev_R }.
\end{align}
We study the convergence of this quantity by considering the numerator and denominator separately (it is easy to show that the fraction converges if both the numerator and denominator converge). By the strong law of large numbers, the following can be concluded for the denominator as $n \rightarrow \infty$:
\begin{equation}
\frac{1}{n} \onev_R^T \onev_R \; \xrightarrow[]{a.s.} \; \int_{\xv \in R} p(\xv)d\xv,
\end{equation}
where $a.s.$ denotes almost sure convergence.
For the numerator, we decompose it into two parts -- a variance term for which we show stochastic convergence using a concentration inequality, and a bias term for which we prove deterministic convergence.

\subsection{Expansion of $\frac{1}{n} \onev_R^T \Lm^m \onev_R$}
\label{sec:expansion_of_V}
Let $V := \frac{1}{n} \onev_R^T \Lm^m \onev_R$. We begin by expanding $V$ as
\begin{align}
V &= \frac{1}{n^{m+1}} \onev_R^T ( \Dm - \Wm)^m \onev_R \nonumber \\
&= \frac{1}{n^{m+1}} \onev_R^T \left( \sum_{k = 0}^{2^m - 1}  \Bm_k \right) \onev_R,
\label{eq:V_expansion}
\end{align}
where $\Bm_k$ denotes the $k^\text{th}$ term out of the $2^m$ terms in the expansion of $(\Dm - \Wm)^m$. $\Bm_k$ is composed of a product of $m$ matrices, each of which can be either $\Dm$ or $-\Wm$. In order to write it down explicitly, one can use the $m$-bit binary representation of the index $k$ and replace $0$s with $\Dm$ and $1$s with $-\Wm$, i.e., if $b_v(k)$ denotes the $v^\text{th}$ most-significant bit in the $m$-bit binary representation of $k$ for $v \in \{1,\dots,m\}$ and $s(k)$ denotes the number of ones in it (i.e., $s(k) := \sum_{v=1}^m b_v(k)$), then
\begin{align}
\Bm_k &= \prod_{v=1}^{m} \left(\Dm^{1- b_v(k)}.(-\Wm)^{b_v(k)} \right) \nonumber \\
&= (-1)^{s(k)} \prod_{v=1}^{m} \left(\Dm^{1- b_v(k)}.\Wm^{b_v(k)} \right),
\end{align}
where the product notation assumes that the ordering of the matrices is kept fixed, i.e., $\prod_{p=1}^{m} \Am_p = \Am_1 \Am_2 \dots \Am_m$.

Noting that $\Dm$ and $\Wm$ are composed of the edge weights $w_{ij} = \frac{1}{(2\pi\sigma^2)^{d/2}}K(\Xm_i,\Xm_j)$, we now describe how to expand the quadratic form $V$ by considering each term $\onev_R^T \Bm_k \onev_R$ individually:

\begin{enumerate}[leftmargin=13pt]
\item
The sign of the term $\onev_R^T \Bm_k \onev_R$ is determined by the number of $(-\Wm)$ matrices in the product $\Bm_k$.
\item 
By using the definitions of $\Dm$ and $\Wm$ in the product expansion of $\Bm_k$, the absolute value of $\onev_R^T \Bm_k \onev_R$ can be expressed through the following template:
\begin{align}
\label{eq:template}
\sum_{i_1,\dots,i_{m+1}} (\onev_R)_{i_1} w_{i_1 i_2} w_{* i_3} 
\dots w_{* i_{m}} w_{* i_{m+1}} (\onev_R)_*,
\end{align}
where $(\onev_R)_i$ denotes the $i^\text{th}$ element of of the indicator vector, and the locations with a ``$*$'' need to be filled with appropriate indices in $\{i_1,\dots,i_{m+1}\}$. Note that the template consists of a product of $m$ edge weights $w_{ij}$, each contributed by either a $\Dm$ or $\Wm$ depending on its location in the expression.
\item By performing an explicit matrix multiplication, we fill the locations from left to right one-by-one using the following rule: let a term containing a $*$ be preceded by an edge-weight $w_{ab}$, then,
\begin{itemize}
\item If $w_{ab}$ is contributed by $\Dm$, then $* = a$.
\item If $w_{ab}$ is contributed by $\Wm$, then $* = b$.
\end{itemize}
\end{enumerate}
Since the binary representation of $k$ is closely tied to the ordering of $\Dm$ and $\Wm$ in the product term $\Bm_k$, we can once again use it to explicitly express $\onev_R^T \Bm_k \onev_R$. In order to populate any ``*'' location according to the rules above, we require a quantity that depends on the position of the last occurring $\Wm$ with respect to any location in the product expression of $\Bm_k$. Therefore, using the $m$-bit binary representation of $k$, we define for location $u \in \{1,\dots,m\}$:
\begin{equation}
c_u(k) := 1 + \max \big( \{0\} \cup \{v | 1 \leq v \leq u, b_v(k) = 1 \} \big),
\label{eq:def_c}
\end{equation}
where $\max(\{.\})$ returns the maximum element in a set of numbers. The template described in~\eqref{eq:template} can then be completed using the rules to obtain
\begin{align}
&\onev_R^T \Bm_k \onev_R \nonumber \\
&= (-1)^{s(k)} \sum_{i_1,\dots,i_{m+1}} \Big[ (\onev_R)_{i_1} w_{i_1 i_2} w_{i_{c_1(k)} i_3} w_{i_{c_2(k)} i_4} \dots \nonumber \\
&\qquad\quad \dots w_{i_{c_{m-2}(k)} i_{m}} w_{i_{c_{m-1}(k)} i_{m+1}} (\onev_R)_{i_{c_{m}(k)}} \Big].
\label{eq:b_k_expression}
\end{align}
Finally, the expansion of $V$ can be obtained by summing the $2^m$ quadratic forms in~\eqref{eq:V_expansion}:
\begin{align}
&V = \frac{1}{n^{m+1}} \sum_{k=0}^{2^m-1} \onev_R^T \Bm_k \onev_R \nonumber \\
&= \frac{1}{n^{m+1}} \sum_{k=0}^{2^m-1} (-1)^{s(k)} \sum_{i_1,\dots,i_{m+1}} (\onev_R)_{i_1} w_{i_1 i_2} w_{i_{c_1(k)} i_3} \dots \nonumber \\
&\qquad\quad \dots w_{i_{c_{m-2}(k)} i_{m}} w_{i_{c_{m-1}(k)} i_{m+1}} (\onev_R)_{i_{c_{m}(k)}} \nonumber \\ 
&= \frac{1}{n^{m+1}} \sum_{i_1,i_2,\dots,i_{m+1}} g\left( \Xm_{i_1},\Xm_{i_2},\dots,\Xm_{i_{m+1}} \right),
\label{eq:V_V_statistic}
\end{align}
where we defined
\begin{align}
&g\left( \Xm_{i_1},\Xm_{i_2},\dots,\Xm_{i_{m+1}} \right) \nonumber \\
&\quad:=  \sum_{k=0}^{2^m-1} (-1)^{s(k)} \Big[ (\onev_R)_{i_1} w_{i_1 i_2} w_{i_{c_1(k)} i_3} w_{i_{c_2(k)} i_4} \dots \nonumber \\
&\qquad\quad \dots w_{i_{c_{m-2}(k)} i_{m}} w_{i_{c_{m-1}(k)} i_{m+1}} (\onev_R)_{i_{c_{m}(k)}} \Big].
\label{eq:g_def}
\end{align}

\subsection{Convergence of variance terms}
For $V = \frac{1}{n} \onev_R^T \Lm^m \onev_R$, we have the following concentration result:
\begin{lemma}[Concentration]
For every $\epsilon > 0$, we have:
\begin{align}
&\Pr{\left( \left| V - \E{V } \right| > \epsilon \right)} \nonumber \\
& \quad \leq 2 \exp{ \left( \frac{-[n/(m+1)]\sigma^{md}\epsilon^2}{2C^m \E{V} + \frac{2}{3}\left| C^m - \sigma^{md} \E{V}\right|\epsilon} \right) },
\label{eq:v_conc}
\end{align}
where $C = 2/(2\pi)^{d/2}$.
\end{lemma}

\begin{proof}
Note that the expansion of $V$ in \eqref{eq:V_V_statistic} has the form of a V-statistic. Further, as defined in \eqref{eq:g_def}, $g$ is composed of a sum of $2^m$ terms, each a product of $m$ kernel functions $K$ that are non-negative. Therefore, we have the following upper bound:
\begin{equation}
g \leq 2^m\|K\|_\infty^m = \left(\frac{2}{(2\pi\sigma^2)^{d/2}}\right)^m = \frac{C^m}{\sigma^{md}}.
\label{eq:g_bound}
\end{equation}
In order to apply a concentration inequality for V, we first re-write it in the form of a U-statistic by regrouping terms in the summation in order to remove repeated indices, as given in \cite{hoeffding_63}:
\begin{align}
V &= \frac{1}{n^{(m+1)}} \sum_{(n,m+1)} g^*\left( \Xm_{i_1},\Xm_{i_2},\dots,\Xm_{i_{m+1}} \right),                         \label{eq:g_star_def}      
\end{align}
where $\sum_{(n,m+1)}$ denotes summation over all ordered (m+1)-tuples $(i_1,\dots,i_{m+1})$ of distinct indices taken from the set $\{1,\dots,n\}$, $n^{(m+1)} = n.(n-1)\dots(n-m)$ is the falling factorial (or number of (m+1)-permutations of $n$) and $g^*$ is a weighted arithmetic mean of specific instances of $g$ that avoids repeating indices:
\begin{align}
&g^*\left( \xv_1,\xv_2,\dots,\xv_{m+1} \right) \nonumber \\
&\quad= \sum_{j=0}^{m+1} \frac{n^{(j)}}{n^{m+1}} \sum\nolimits_{(j)}^* g\left( \xv_{l_1},\xv_{l_2},\dots,\xv_{l_{m+1}} \right),
\end{align}
where $\sum_{(j)}^*$ denotes summation over all $(m+1)$-tuples $(l_1,l_2,\dots,l_{m+1})$ formed from $\{1,\dots,j\}$ with exactly $j$ distinct indices. Note that the number of such $(m+1)$-tuples is given by $\genfrac{\{}{\}}{0pt}{}{m+1}{j}$, which is a Stirling number of the second kind. Hence, we have
\begin{equation}
\|g_*\|_\infty \leq \sum_{j=0}^{m+1} \frac{n^{(j)}}{n^{m+1}} \genfrac{\{}{\}}{0pt}{}{m+1}{j}\|g\|_\infty = \|g\|_\infty,
\end{equation}
where we used the property $\sum_{j=0}^{m+1} n^{(j)} \genfrac{\{}{\}}{0pt}{}{m+1}{j} = n^{m+1}$. Therefore, $g^*$ has the same upper bound as that of $g$ derived in \eqref{eq:g_bound}. Moreover, using the fact that $\E{V} = \E{g^*(\Xm_{i_1},\Xm_{i_2},\dots,\Xm_{i_{m+1}})}$, we can bound the variance of $g^*$ as
\begin{equation}
\Var{g^*} 
\leq \|g^*\|_\infty \E{g^*} = \frac{C^m}{\sigma^{md}} \E{V}.
\end{equation}
Finally, plugging in the bound and variance of $g^*$ in Bernstein's inequality for U-statistics as stated in~\cite{hoeffding_63, hein_thesis_06}, we arrive at the desired result of \eqref{eq:v_conc}.
\end{proof}
Note that as $n \rightarrow \infty$ and $\sigma \rightarrow 0$ with rates satisfying $(n \sigma^{md}) / (mC^m) \rightarrow \infty$, we have $P(|V - \E{V}| > \epsilon) \rightarrow 0$ for all $\epsilon > 0$. The continuous mapping theorem then allows us to conclude that $V^{1/m} \xrightarrow[]{p.} (\E{V})^{1/m}$.

\subsection{Expansion of $\E{\frac{1}{n}\onev_R^T \Lm^m \onev_R}$}
The V-statistic expansion of $V = \frac{1}{n} \onev_R^T \Lm^m \onev_R$ in \eqref{eq:V_V_statistic} has summands with repeating indices, hence we first define a U-statistic counterpart that avoids these repetitions:
\begin{equation}
U := \frac{1}{n^{(m+1)}} \sum_{(n,m+1)} g(\Xm_{i_1}, \Xm_{i_2}, \dots, \Xm_{i_{m+1}}),
\end{equation}
where $g(\Xm_{i_1}, \Xm_{i_2}, \dots, \Xm_{i_{m+1}})$ are the kernels defined in~\eqref{eq:g_def}, and the definitions of $\sum_{(n,m+1)}$ and $n^{(m+1)}$ are the same as those for~\eqref{eq:g_star_def}. The $U$-statistic definition is convenient since
\begin{equation}
\E{U} = \E{g(\Xm_{i_1}, \Xm_{i_2}, \dots, \Xm_{i_{m+1}})},
\end{equation}
as opposed to $\E{V}$, where one would have to deal with terms with repeated indices separately. Further, note that
\begin{equation}
n^{m+1}V = n^{(m+1)} U + \sum_{(n,m+1)^*} g(\Xm_{i_1}, \Xm_{i_2}, \dots, \Xm_{i_{m+1}}),
\end{equation}
where $\sum_{(n,m+1)^*}$ denotes summation over all ordered $(m+1)$-tuples $(i_1,\dots,i_{m+1})$ of indices obtained from $\{1,2,\dots,n\}$ such that at least two of them are equal. Note that there are $n^{m+1} - n^{(m+1)}$ terms in the summation $\sum_{(n,m+1)^*}$. Therefore, we have
\begin{align}
\E{V} &= \frac{n^{(m+1)}}{n^{m+1}} \E{U} + \frac{1}{n^{m+1}} \sum_{(n,m+1)^*} g(\Xm_{i_1}, \dots, \Xm_{i_{m+1}}) \nonumber \\
&= \E{g} + \frac{n^{m+1} - n^{(m+1)}}{n^{m+1}}.\E{g} \nonumber \\
&\qquad\quad+ \frac{1}{n^{m+1}}\sum_{(n,m+1)^*} g(\Xm_{i_1}, \dots, \Xm_{i_{m+1}}) \nonumber \\
&= \E{g} + O\left( \frac{m^2C^m}{n\sigma^{md}} \right),
\label{eq:exp_V_exp_g_relation}
\end{align}
where we used $n^{m+1} - n^{(m+1)} = O(m^2 n^{m})$, $\E{g} \leq \|g\|_\infty$ and $\|g\|_\infty = \frac{C^m}{\sigma^{md}}$ from \eqref{eq:g_bound}.

We now focus on computing $\E{g(\Xm_{i_1}, \Xm_{i_2}, \dots, \Xm_{i_{m+1}})}$. Based on~\eqref{eq:g_def}, we can express it as follows:
\begin{equation}
\E{g(\Xm_{i_1}, \Xm_{i_2}, \dots, \Xm_{i_{m+1}})} = \sum_{k=0}^{2^m-1} h_k,
\label{eq:exp_g}
\end{equation}
where we define:
\begin{align}
h_k &:= (-1)^{s(k)} \int_{\xv_1} \int_{\xv_2} \dots \int_{\xv_{m+1}}  \nonumber \\
&\Big[ 1_R(\xv_1) K\left(\xv_1, \xv_2\right) K\left(\xv_{c_1(k)}, \xv_3\right) K\left(\xv_{c_2(k)}, \xv_4\right) \dots \nonumber \\
&\quad \dots K\left(\xv_{c_{m-2}(k)}, \xv_{m}\right) K\left(\xv_{c_{m-1}(k)}, \xv_{m+1}\right) 1_R(\xv_{c_m(k)}) \Big] \nonumber \\
&\qquad\quad p(\xv_1)d\xv_1 p(\xv_2) d\xv_2 \dots p(\xv_{m+1}) d\xv_{m+1},
\label{eq:exp_h_k}
\end{align}
with $c_u(k)$ defined as in~\eqref{eq:def_c}.

\subsection{Convergence of bias term for the separable model}
To evaluate the convergence of bias terms, we shall require the following properties of the $d$-dimensional Gaussian kernel:
\begin{lemma}
\label{lem:kernelone}
If $p(\xv)$ is twice differentiable, then
\begin{align}
\int K_{\sigma^2}(\xv,\yv)p(\yv)d\yv = p(\xv) + O\left( \sigma^2 \right).
\label{eq:K_int1}
\end{align}
\end{lemma}
\begin{proof}
Using the substitution $\yv = \xv + \tv$ followed by a Taylor series expansion about $\xv$, we have
\begin{align}
\int &K_{\sigma^2}(\xv,\yv)p(\yv)d\yv \nonumber \\
&= \int \frac{1}{(2\pi\sigma^2)^{d/2}} e^{-\|\tv\|^2/2\sigma^2} p(\xv + \tv) d\tv \nonumber \\
&= \int \frac{1}{(2\pi\sigma^2)^{d/2}} e^{-\|\tv\|^2/2\sigma^2} \bigg( p(\xv) + \tv^T {\bm \nabla} p(\xv) \nonumber \\
&\quad\quad + \frac{1}{2} \tv^T {\bm \nabla}^2 p(\xv) \tv + \dots \bigg) d\tv \nonumber \\
&= p(\xv) + 0 + \frac{\sigma^2}{2} {\rm Tr}({\bm \nabla}^2 p(\xv)) + \dots \nonumber \\
&= p(\xv) + O(\sigma^2), \nonumber
\end{align}
where ${\rm Tr}(.)$ denotes the trace of a matrix, and the third step follows from simple component-wise integration.
\end{proof}

\begin{lemma}
\label{lem:kerneltwo}
If $p(\xv)$ is twice differentiable, then
\begin{align}
&\int K_{a\sigma^2}(\xv,\zv) K_{b\sigma^2}(\zv,\yv) p(\zv) d\zv \nonumber \\
&\quad = K_{(a+b)\sigma^2}(\xv,\yv) \; \left( p\left( \frac{b \xv + a \yv}{a + b} \right) + O\left( \sigma^2 \right) \right).
\label{eq:K_int2}
\end{align}
\end{lemma}
\begin{proof}
Note that
\begin{align}
&K_{a\sigma^2}(\xv,\zv) K_{b\sigma^2}(\zv,\yv) \nonumber \\
&\quad=\frac{1}{(2\pi a\sigma^2)^{\frac{d}{2}}} e^{-\frac{\|\xv -\zv\|^2}{2a\sigma^2}} \frac{1}{(2\pi b\sigma^2)^\frac{d}{2}} e^{-\frac{\|\zv -\yv\|^2}{2b\sigma^2}} \nonumber \\
&\quad=\frac{1}{(2\pi(a+b)\sigma^2)^\frac{d}{2}} e^{-\frac{\|\xv -\yv\|^2}{2(a+b)\sigma^2}} \frac{1}{(2\pi\frac{ab}{a+b}\sigma^2)^\frac{d}{2}} e^{-\frac{\|\zv - \frac{b\xv + a\yv}{a + b} \|^2}{2\left(\frac{ab}{a+b}\right)\sigma^2}} \nonumber \\
&\quad= K_{(a+b)\sigma^2}(\xv,\yv) \;\; K_{\frac{ab}{a+b}\sigma^2}\left(\frac{b\xv + a\yv}{a + b},\zv \right). \nonumber
\end{align}
Therefore, we have
\begin{align}
&\int K_{a\sigma^2}(\xv,\zv) K_{b\sigma^2}(\zv,\yv) p(\zv) d\zv \nonumber \\
&\quad = K_{(a+b)\sigma^2}(\xv,\yv) \int K_{\frac{ab}{a+b}\sigma^2}\left(\frac{b\xv + a\yv}{a + b}, \zv \right) p(\zv) d\zv \nonumber \\
&\quad = K_{(a+b)\sigma^2}(\xv,\yv) \; \left( p\left( \frac{b \xv + a \yv}{a + b} \right) + O\left( \sigma^2 \right) \right), \nonumber
\end{align}
where the last step follows from Lemma~\ref{lem:kernelone}.
\end{proof}
In order to prove convergence for the separable model, we need the following results:
\begin{lemma}
If $p(\xv)$ is Lipschitz continuous, then for a smooth hypersurface $\partial S$ that divides $\mathbb{R}^d$ into $S_1$ and $S_2$, and whose curvature has radius lower-bounded by $\tau > 0$,
\begin{align}
&\lim_{\sigma \rightarrow 0} \frac{1}{\sigma} \int_{S_1} \int_{S_2} K_{\sigma^2}(\xv_1,\xv_2) p^\alpha(\xv_1) p^\beta(\xv_2) d\xv_1 d\xv_2 \nonumber \\
&\quad \quad = \frac{1}{\sqrt{2\pi}}\int_{\partial S} p^{\alpha + \beta}(\sv) d\sv,
\label{eq:flow_one}
\end{align}
where $\alpha$ and $\beta$ are positive integers. Moreover, for positive integers $a,b$, and $\alpha, \beta, \alpha', \beta'$ such that $\alpha + \beta = \alpha'+\beta' = \gamma$, we have: 
\begin{align}
&\lim_{\sigma \rightarrow 0} \frac{1}{\sigma} \int_{S_1} \int_{S_1} \Big[ K_{a\sigma^2}(\xv_1,\xv_2) p^\alpha(\xv_1) p^\beta(\xv_2) \nonumber \\
&\quad \quad \quad \quad \quad \quad - K_{b\sigma^2}(\xv_1,\xv_2) p^{\alpha'}(\xv_1) p^{\beta'}(\xv_2) \Big] d\xv_1 d\xv_2 \nonumber \\
&\quad \quad \quad \quad \quad \quad \quad \quad = \frac{\sqrt{b}-\sqrt{a}}{\sqrt{2\pi}}\int_{\partial S} p^\gamma(\sv) d\sv.
\label{eq:flow_two}
\end{align}
\label{lemma:flow}
\end{lemma}
\begin{proof}
See Appendix~\ref{app:proof_lemma_5}. 
\end{proof}
We now prove the deterministic convergence of $\E{\frac{1}{n}\onev_S^T \Lm^m \onev_S}$ in the following lemma:
\begin{lemma}
\label{lem:bias_separable}
As $n \rightarrow \infty$, $\sigma \rightarrow 0$ such that $m2^m\sigma \rightarrow 0$ and $\frac{m^2C^m}{n\sigma^{md+1}}\rightarrow 0$, we have
\begin{equation}
\frac{1}{\sigma}\E{\frac{1}{n}\onev_S^T \Lm^m \onev_S} \rightarrow \frac{t(m)}{\sqrt{2\pi}} \int_{\partial S} p^{m+1}(\sv)d\sv,
\label{eq:conv_bias_sep}
\end{equation}
where $t(m) = \sum_{r = 0}^{m-1} \binom{m-1}{r} (-1)^r (\sqrt{r+1} - \sqrt{r})$.
\end{lemma}
\begin{proof}
Using~\eqref{eq:exp_V_exp_g_relation} and~\eqref{eq:exp_g}, and replacing $\onev_R$ with $\onev_S$, we have
\begin{equation}
\frac{1}{\sigma}\E{\frac{1}{n} \onev_S^T \Lm^m \onev_S} = \frac{1}{\sigma} \sum_{k=0}^{2^m-1} h_k + O\left( \frac{m^2C^m}{n\sigma^{md+1}} \right).
\label{eq:exp_sep}
\end{equation}
We pair all even-indexed and odd-indexed terms together to rewrite the summation as:
\begin{equation}
\sum_{k=0}^{2^m-1} h_k = \sum_{l=0}^{2^{m-1}-1} (h_{2l} + h_{2l+1}).
\label{eq:h_pairwise_sum}
\end{equation}
Now, $h_0$ and $h_1$ can be evaluated by repeatedly applying \eqref{eq:K_int1} for every Gaussian kernel in the definition from~\eqref{eq:exp_h_k}. Hence, for the first summation pair, we obtain:
\begin{align}
h_0 + h_1 &= \int_{S} \int_{\mathbb{R}^d} K_{\sigma^2}(\xv,\yv) p^m(\xv) p(\yv) d\xv d\yv \nonumber \\
&\quad - \int_{S} \int_S K_{\sigma^2}(\xv,\yv) p^m(\xv) p(\yv) d\xv d\yv + O(m\sigma^2) \nonumber \\
&= \int_{S} \int_{S^c} K_{\sigma^2}(\xv,\yv) p^m(\xv) p(\yv) d\xv d\yv + O(m\sigma^2).
\label{eq:sep_1st_term}
\end{align}
For the rest of the terms, we also require the use of~\eqref{eq:K_int2}. However, in this case, we encounter several terms of the form $p(\theta \xv + (1-\theta) \yv)$ for some $\theta \in [0,1]$. Since $m\sigma^2 \rightarrow 0$ and $p(\xv)$ is assumed to be Lipschitz continuous, we can approximate such terms by $p(\xv)$ or $p(\yv)$. Further, the number of times we have to apply~\eqref{eq:K_int2} in any $h_k$ is equal to the number of occurrences of $\Wm$ in $\Bm_k$ (which is $s(k)$). Therefore, for $1 \leq l \leq 2^{m-1}-1$, we have
\begin{align}
&h_{2l} + h_{2l+1} \nonumber \\
&\quad= (-1)^{s(2l)}\bigg[ \int_S \int_S K_{s(2l)\sigma^2}(\xv,\yv) p^\alpha(\xv) p^\beta(\yv) d\xv d\yv \nonumber \\
&\quad- \int_S \int_S K_{s(2l+1)\sigma^2}(\xv,\yv) p^{\alpha'}(\xv) p^{\beta'}(\yv) d\xv d\yv \bigg] + O(m\sigma^2),
\label{eq:sep_rth_term}
\end{align}
where $\alpha, \beta, \alpha', \beta'$ are positive integers such that $\alpha + \beta = \alpha' + \beta' = m+1$. Plugging~\eqref{eq:sep_1st_term} and~\eqref{eq:sep_rth_term} into~\eqref{eq:exp_sep}, we get:
\begin{align}
&\frac{1}{\sigma}\E{\frac{1}{n} \onev_S^T \Lm^m \onev_S} \nonumber \\
&= \frac{1}{\sigma} \int_{S} \int_{S^c} K_{\sigma^2}(\xv,\yv) p^m(\xv) p(\yv) d\xv d\yv \nonumber \\
&\quad+ \sum_{r=1}^{m-1} \binom{m-1}{r} (-1)^r \nonumber \\
&\qquad\qquad \frac{1}{\sigma} \bigg[\int_S \int_S K_{r\sigma^2}(\xv,\yv) p^\alpha(\xv) p^\beta(\yv) d\xv d\yv \nonumber \\
&\qquad\qquad\quad- \int_S \int_S K_{(r+1)\sigma^2}(\xv,\yv) p^{\alpha'}(\xv) p^{\beta'}(\yv) d\xv d\yv \bigg] \nonumber \\
&\quad + O(m2^m\sigma) + O\left( \frac{m^2C^m}{n\sigma^{md+1}} \right),
\label{eq:sep_exp}
\end{align}
where we grouped terms based on $r = s(2l)$ in the summation (note that there are $\binom{m-1}{r}$ for a given $r$).

Using Lemma \ref{lemma:flow}, we conclude that the right hand side of \eqref{eq:sep_exp} converges as $n\rightarrow \infty$ and $\sigma \rightarrow 0$ to 
\begin{equation*}
\frac{1}{\sqrt{2\pi}}\int_{\partial S} p^{m+1}(\sv) d\sv + \sum_{r=1}^{m-1} \frac{\sqrt{r+1}-\sqrt{r}}{\sqrt{2\pi}}\int_{\partial S} p^{m+1}(\sv) d\sv,
\end{equation*}
which is the desired result.
\end{proof}

\noindent Using the continuous mapping theorem on~\eqref{eq:conv_bias_sep}, we can conclude
\begin{align}
\left( \frac{1}{\sigma} \E{\frac{1}{n}\onev_S^T \Lm^m \onev_S} \right)^{1/m}  \rightarrow \left( \frac{t(m)}{\sqrt{2\pi}} \int_{\partial S} p^{m+1}(\sv)d\sv \right)^{1/m}.
\end{align}
Finally, we note that as $m \rightarrow \infty$, we have
\begin{equation}
\left( \frac{\frac{t(m)}{\sqrt{2\pi}} \int_{\partial S} p^{m+1}(\sv)d\sv}{ \int_S p(\xv)d\xv} \right)^{1/m}
\longrightarrow
\sup_{\sv \in \partial S} p(\sv).
\end{equation}
Therefore, we conclude for the separable model
\begin{equation}
\frac{1}{\sigma^{1/m}} \omega_m(\onev_S) \rightarrow \sup_{\sv \in \partial S} p(\sv).
\end{equation}
\subsection{Convergence of bias term for the nonseparable model}

For the nonseparable model, we need to prove convergence of $\E{\frac{1}{n}\onev_A^T \Lm^m \onev_A}$. This is illustrated in the following lemma:
\begin{lemma}
\label{lem:bias_nonseparable}
As $n \rightarrow \infty$, $\sigma \rightarrow 0$ such that $m2^m\sigma^2 \rightarrow 0$ and $\frac{m^2C^m}{n\sigma^{md}}\rightarrow 0$, we have
\begin{equation}
\E{\frac{1}{n}\onev_A^T \Lm^m \onev_A} \rightarrow \int \alpha_A \alpha_{A^c} p_A(\xv) p_{A^c}(\xv) p^{m-1}(\xv) d\xv.
\end{equation}
\end{lemma}
\begin{proof}
Similar to the proof of Lemma~\ref{lem:bias_separable}, we use~\eqref{eq:exp_V_exp_g_relation} and~\eqref{eq:exp_g}, and replace $\onev_R$ with $\onev_A$ to obtain
\begin{equation}
\E{\frac{1}{n} \onev_A^T \Lm^m \onev_A} = \sum_{l=0}^{2^{m-1}-1} \left( h_{2l} + h_{2l+1} \right) + O\left( \frac{m^2C^m}{n\sigma^{md}} \right).
\label{eq:exp_nonsep}
\end{equation}
Using~\eqref{eq:K_int1} repeatedly in the definition~\eqref{eq:exp_h_k}, we get
\begin{align}
h_0 + h_1 &= \int \alpha p_A(\xv) p^{m}(\xv) d\xv  \nonumber \\
&\quad - \int \left( \alpha p_A(\xv) \right)^2 p^{m-1}(\xv) d\xv + O(m\sigma^2) \nonumber \\
&= \int \alpha_A \alpha_{A^c} p_A(\xv) p_{A^c}(\xv) p^{m-1}(\xv) d\xv + O(m\sigma^2),
\label{eq:nonsep_first_pair}
\end{align}
where we used the fact that $p(\xv) = \alpha_A p_A(\xv) + \alpha_{A^c} p_{A^c}(\xv)$.
Similarly, for $1 \leq l \leq 2^{m-1}-1$, we have
\begin{align}
h_{2l} + h_{2l+1} = O\left(m\sigma^2 \right). 
\label{eq:nonsep_remaining_pairs}
\end{align}
Putting together~\eqref{eq:nonsep_first_pair} and~\eqref{eq:nonsep_remaining_pairs} into~\eqref{eq:exp_nonsep}, we get
\begin{align}
\E{\frac{1}{n} \onev_A^T \Lm^m \onev_A} &= \int \alpha_A \alpha_{A^c} p_A(\xv) p_{A^c}(\xv) p^{m-1}(\xv) d\xv \nonumber \\
&\quad + O(m2^m\sigma^2) + O\left( \frac{m^2C^m}{n\sigma^{md}} \right).
\end{align}
Taking limits while satisfying the stated rate conditions, we get the desired result.
\end{proof}
\noindent We finally note that as $m \rightarrow \infty$, we have
\begin{equation}
\left( \frac{\int \alpha_A \alpha_{A^c} p_A(\xv) p_{A^c}(\xv) p^{m-1}(\xv) d\xv}{\int_A p(\xv)d\xv} \right)^{1/m} \xrightarrow{\;\;\text{s.}\;} \sup_{\xv \in \partial A} p(\xv).
\end{equation}
Therefore, we conclude for the nonseparable model
\begin{equation}
\omega_m(\onev_A) \rightarrow \sup_{\xv \in \partial A} p(\xv).
\end{equation}
Note that Lemma~\ref{lem:bias_nonseparable} for the special case of $m = 1$ yields
\begin{equation}
\frac{1}{n} \onev_A^T \Lm \onev_A \to \int \alpha_A \alpha_{A^c} p_A(\xv) p_{A^c}(\xv) d\xv,
\end{equation}
which proves Theorem~\ref{thm:cut_nonsep}.
\subsection{Proof of Theorem~\ref{thm:esd}}
\label{sec:proof_esd}
We begin by recalling the definition of the empirical spectral distribution (ESD) of $\Lm$:
\begin{equation}
\mu_n(x) := \frac{1}{n} \sum_{i=1}^n \delta(x - \lambda_i),
\end{equation}
where $\{\lambda_i\}$ are the eigenvalues of $\Lm$.
For each $x$, $\mu_n(x)$ is a function of $\Xm_1, \dots, \Xm_n$, and thus a random variable. 
%
Note that the fraction of eigenvalues of $\Lm$ below a constant $t$, and its expected value can be computed from the ESD as
\begin{align}
\frac{1}{n} \mathcal{N}_\Lm(t) &= \int_0^t \mu_n(x) dx, \\
\E{\frac{1}{n} \mathcal{N}_\Lm(t)} &= \int_0^t \E{\mu_n(x)} dx.
\end{align}
Therefore, to understand the behavior of the expected fraction of eigenvalues of $\Lm$ below $t$, we need to analyze the convergence of the expected ESD in the asymptotic limit.
The idea is to show the convergence of the moments of $\E{\mu_n(x)}$ to the moments of a limiting distribution $\mu(x)$. Then, by a standard convergence result, $\E{\mu_n(I)} \rightarrow \mu(I)$ for intervals $I$. More precisely, let the $\Rightarrow$ symbol denote weak convergence of measures, then we use the following result that follows from the Weierstrass approximation theorem:
\begin{lemma}
Let $\mu_n$ be a sequence of probability measures and $\mu$ be a compactly supported probability measure. If $\int x^m \mu_n(dx) \rightarrow \int x^m \mu(dx)$ for all $m \geq 1$, then $\mu_n \Rightarrow \mu$.
\end{lemma} 
We then use the following result on equivalence of different notions of weak convergence of measures~\cite[Theorem 25.2]{billingsley_95} in order to prove our result for cumulative distribution functions.
\begin{lemma}
$\mu_n \Rightarrow \mu$ if and only if $\mu_n(A) \rightarrow \mu(A)$ for every $\mu$-continuity set $A$.
\end{lemma}
Therefore, we simply need to analyze the convergence of moments of $\E{\mu_n(x)}$.
Note that the $m^\text{th}$ moment of $\E{\mu_n(x)}$ can be written as:
\begin{equation}
\int x^m \E{\mu_n(x)} dx = \frac{1}{n} \sum_{i=1}^{n} \E{\lambda_i^m} =  \E{ \frac{1}{n}\Tr{\Lm^m} }.
\label{eq:mom_exp_esd}
\end{equation}
We reuse our analysis in Section~\ref{sec:expansion_of_V}, specifically the expansion in~\eqref{eq:V_expansion} to obtain
\begin{equation}
\frac{1}{n} \Tr{\Lm^m} = \frac{1}{n} \Tr{\frac{1}{n^m} (\Dm-\Wm)^m} = \frac{1}{n^{m+1}} \sum_{k=0}^{2^m-1} \Tr{\Bm_k}.
\end{equation}
Using the binary representation of $k$ once again similar to~\eqref{eq:b_k_expression}, we can compute:
\begin{align}
\Tr{\Bm_0} &= \sum_{i_1,\dots,i_{m+1}} \Big[ w_{i_1 i_2} w_{i_1 i_3} w_{i_1 i_4} \dots  w_{i_1 i_{m}} w_{i_1 i_{m+1}} \Big], \nonumber \\
\Tr{\Bm_1} &= \sum_{i_1,\dots,i_{m}} \Big[ w_{i_1 i_2} w_{i_1 i_3} w_{i_1 i_4} \dots  w_{i_1 i_{m}} w_{i_m i_1} \Big], \nonumber \\
&\quad\vdots \nonumber \\
\Tr{\Bm_k} &= \sum_{i_1,\dots,i_{m}} \Big[ w_{i_1 i_2} w_{i_{c_1(k)} i_3} w_{i_{c_2(k)} i_4} \dots  \nonumber \\
&\qquad\qquad\qquad \dots w_{i_{c_{m-2}(k)} i_{m}} w_{i_{c_{m-1}(k)} i_1} \Big].
\end{align}
Note that $\Tr{\Bm_k}$ has a summation over $m$ indices for $k>1$, as a result, a factor of $\frac{1}{n}$ remains in the expectation.
Similarly, terms with repeated indices disappear and thus, we have the following for the right hand side of \eqref{eq:mom_exp_esd} as $n\rightarrow \infty$:
\begin{align}
&\E{ \frac{1}{n}\Tr{\Lm^m} } \nonumber \\
&= \bigg[ \int \left( \int K(\xv_{1},\xv_{2}) p(\xv_{2}) d\xv_{2} \right) \dots \nonumber \\
&\quad \dots \left( \int K(\xv_{1},\xv_{m+1}) p(\xv_{m+1}) d\xv_{m+1} \right) p(\xv_{1}) d\xv_{1} \bigg].
\end{align}
Using \eqref{eq:K_int1} repeatedly in the equation above, we get:
\begin{align}
\E{ \frac{1}{n}\Tr{\Lm^m} } &= \int p^{m+1}(\xv)d\xv + O\left(m\sigma^2 \right).
\end{align}
Therefore, as $n \rightarrow \infty$ and $\sigma \rightarrow 0$, we have:
\begin{align}
\int x^m \E{\mu_n(x)} dx \rightarrow \int p^{m}(\xv)p(\xv)d\xv.
\end{align}
From the right hand side of the equation above, we conclude that the $m^\text{th}$ moment of the expected ESD of $\Lm$ converges to the $m^\text{th}$ moment of the distribution of a random variable $Y = p(\Xm)$, where $p(\xv)$ is the probabilty density function of $\Xm$.
Moreover, since $p_Y(y)$ has compact support, $\E{\mu_n(x)}$ converges weakly to the probability density function of $p_Y(y)$. Hence, the following can be said about the expected fraction of eigenvalues of $\Lm$:
\begin{align}
&\E{\frac{1}{n} \mathcal{N}_\Lm(t)} = \int_0^t \E{\mu_n(x)}dx \nonumber \\
&\quad\quad\quad\quad \xrightarrow{\;\; \text{s.} \;} \int_0^t p_Y(y) dy = \int_{p(\xv)\leq t} p(\xv)d\xv.
\end{align}
This proves our claim in Theorem~\ref{thm:esd}. Note that, to prove the stochastic convergence of the fraction itself rather than its expected value, we would need a condition similar to those in Theorems~\ref{thm:conv_sep} and~\ref{thm:conv_nonsep} to hold for each moment. In that case, $\sigma$ will go to 0 in a prohibitively slow fashion. We believe that this is an artifact of the methods we employ for proving the result. Hence, our conjecture is that the convergence result holds for $\frac{1}{n} \mathcal{N}_\Lm(t)$ itself, and we leave the analysis of this statement for future work.

\section{Numerical validation}
\label{sec:exp}

We now present simple numerical experiments\footnote{Link to code: \url{https://github.com/aamiranis/asymptotics_graph_ssl}} to validate our results and demonstrate their usefulness in practice.
A key focus in our experiments is to confirm Conjecture~\ref{conjecture:bw}, i.e., the convergence results for the bandwidth estimates also hold for the actual bandwidths. In order to achieve this, we work directly with the bandwidths of the indicators instead of their estimates and numerically validate their convergence for both the separable and nonseparable models.

For simulating the separable model, we first consider a data distribution based on a 2D Gaussian Mixture Model (GMM) with two Gaussians: $\mu_1 = [-1, \; 0], \Sigma_1 = 0.25\Id$ and $\mu_2 = [1 ,\; 0], \Sigma_2 = 0.16\Id$, and mixing proportions $\alpha_1 = 0.4$ and $\alpha_2 = 0.6$ respectively. 
The probability density function is illustrated in Figure~\ref{fig:density}. 
Next, we evaluate the claim of Theorem~\ref{thm:conv_sep} on five boundaries, described in Table~\ref{tab:boundaries}. 
These boundaries are depicted in Figure~\ref{fig:boundaries} and are illustrative of typical separation assumptions such as linear or non-linear and low or high density.

\begin{table}[h]
\caption{Illustrative boundaries used in the separable model.} \label{tab:boundaries}
\vspace{0.1in}
\begin{center}
\begin{tabular*}{0.9\linewidth}{@{\extracolsep{\fill} }  c  c  c  }
\hline
Boundary         & Description     & $\sup_{\sv \in \partial S} p(\sv)$ \\
\hline
$\partial S_1$   & $x = 0$         &  0.0607 \\
$\partial S_2$   & $x = -1$        &  0.2547 \\
$\partial S_3$   & $x = y^2 - 1$   &  0.2547 \\
$\partial S_4$   & $y = 0$         &  0.5969 \\
$\partial S_5$   & $x^2 + y^2 = 1$ &  0.5969 \\
\hline
\end{tabular*}
\end{center}
\end{table}

For simulating the nonseparable model, we first construct the following smooth (twice-differentiable) 2D probability density function
\begin{align}
q(x,y) = \begin{cases} \frac{3}{\pi} \left[ 1 - (x^2 + y^2) \right]^2, \quad & x^2 + y^2 \leq 1 \\ 0, \quad & x^2 + y^2 > 1 \end{cases}
\end{align}
Note that data points $(X,Y)$ can be sampled from this distribution by setting the coordinates $X = \sqrt{1 - U^{1/4}} \cos(2\pi V)$, $Y = \sqrt{1 - U^{1/4}} \sin(2\pi V)$, where $U,V \sim \text{Uniform}(0,1)$. 
We then use $q(x,y)$ to define a nonseparable 2D model with mixture density $p(x,y) = \alpha_A p_A(x,y) + \alpha_{A^c} p_{A^c}(x,y)$, where $p_A(x,y) = q(x-0.75,y)$, $p_{A^c}(x,y) = q(x+0.75,y)$ and $\alpha_A = \alpha_{A^c} = 0.5$. The probability density function is illustrated in Figure~\ref{fig:density}. The overlap region or boundary $\partial A$ for this model is given by 
\begin{align}
\partial A &= \big\{ (x,y) : (x-0.75)^2 + y^2 < 1 \nonumber \\
&\qquad\qquad\qquad \text{ and } (x+0.75)^2 + y^2 < 1 \big\}.
\end{align}
Further, for this model, we have $\sup_{\partial A} p(\xv) = 0.2517$.
\begin{figure}[t]
\begin{center}
\begin{subfigure}{0.99\linewidth}
\centering
\includegraphics[width=0.75\linewidth]{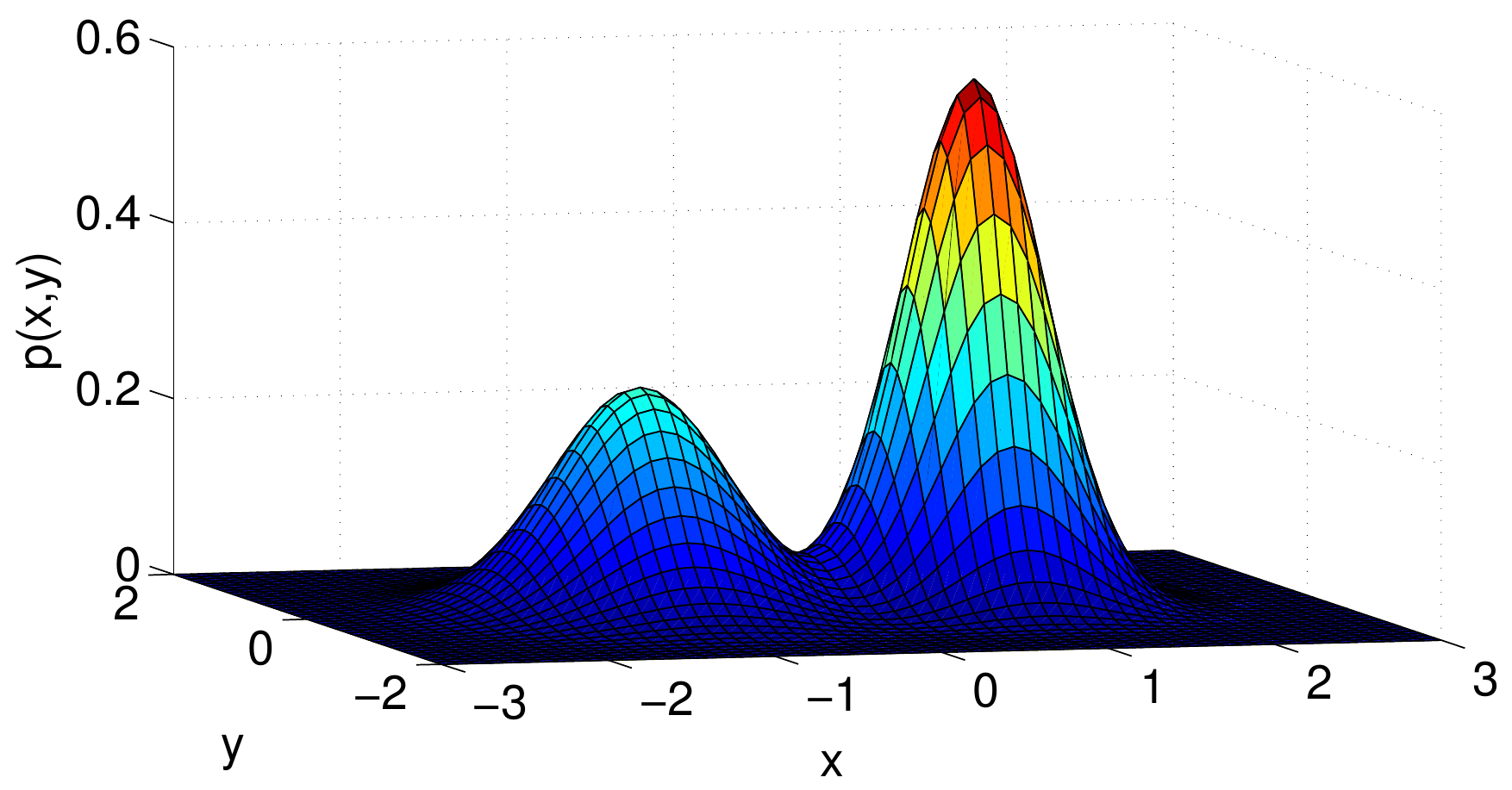}
\caption{}
\end{subfigure}
\\
\begin{subfigure}{0.99\linewidth}
\centering
\includegraphics[width=0.75\linewidth]{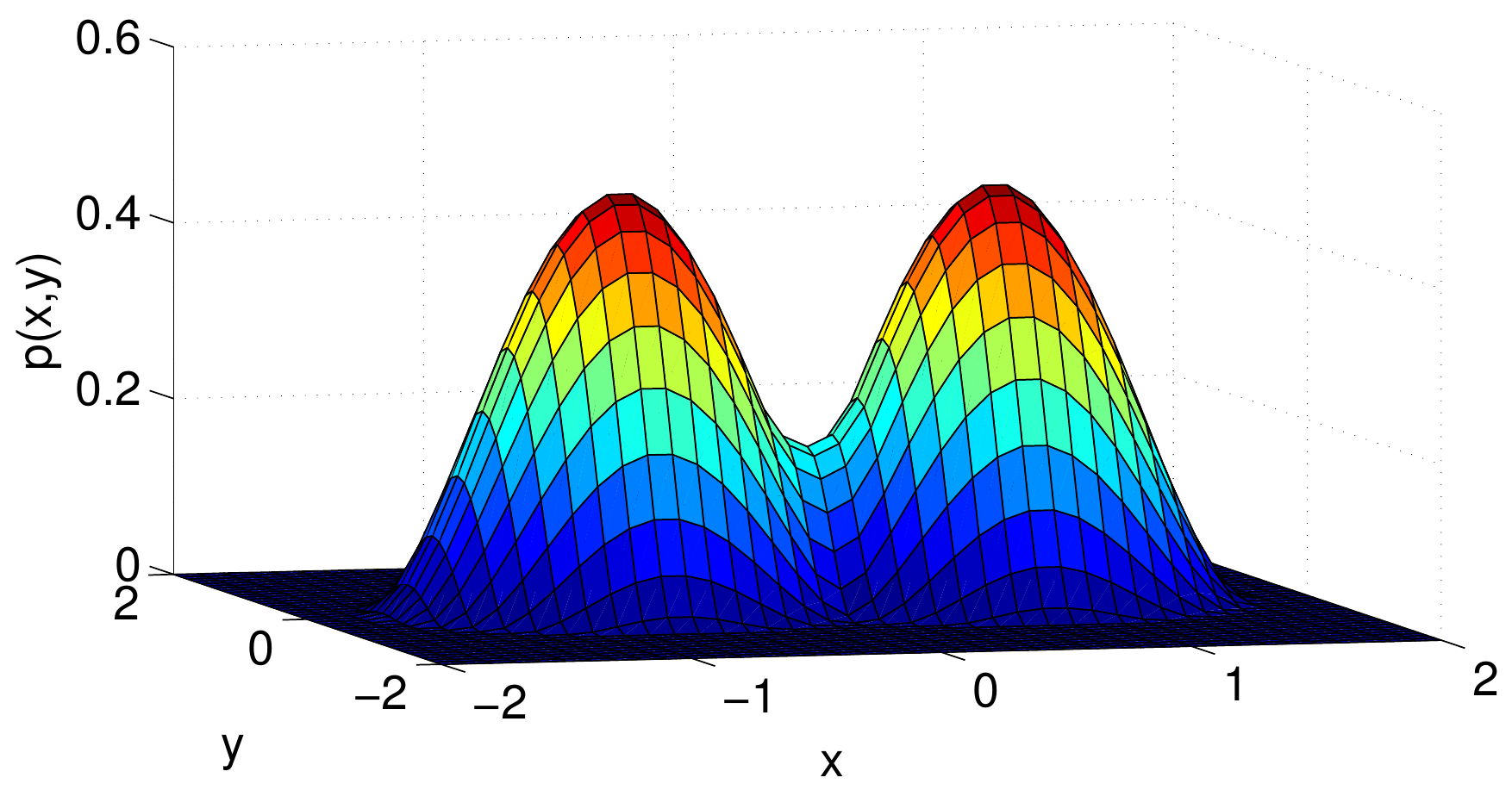}
\caption{}
\end{subfigure}
\end{center}
\caption{Probability density functions to generate data for (a) separable model, (b) nonseparable model.}
\label{fig:density}
\end{figure}
\begin{figure}[t]
\centerline{\includegraphics[width=0.55\linewidth]{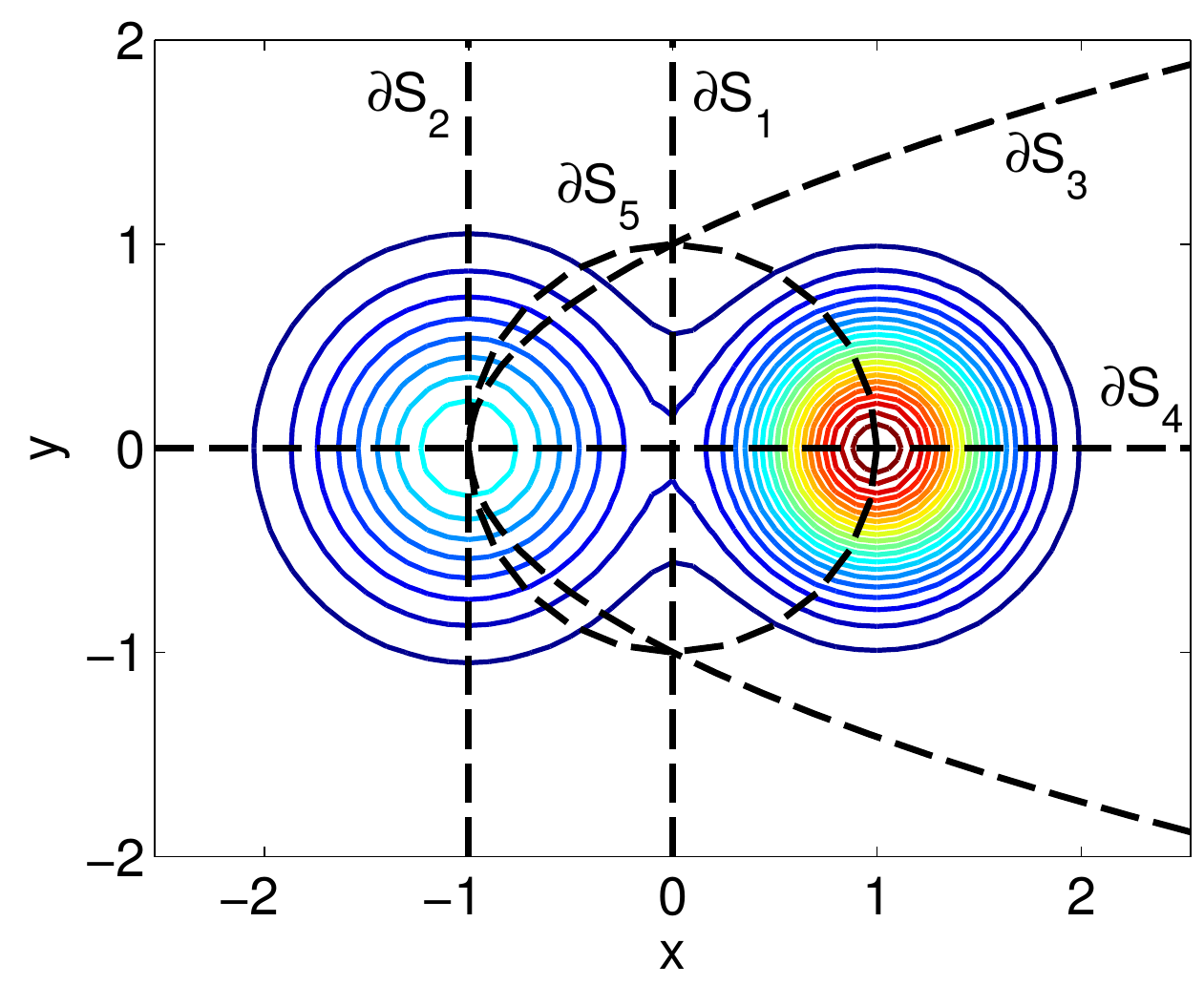}}
\caption{Boundaries $\{ \partial S_i\}$ considered in the separable model.}
\label{fig:boundaries}
\end{figure}
\par
In our first experiment, we validate the statements of Theorems~\ref{thm:conv_sep} and~\ref{thm:conv_nonsep} by comparing the left and right hand sides of \eqref{eq:bwresult_sep} and \eqref{eq:bwresult_nonsep} for corresponding boundaries. 
This is carried out in the following way: we draw $n = 2500$ points from each model and construct the corresponding similarity graphs using $\sigma = 0.1$.
Then, for the boundaries $\partial S_i$ in the separable model and $\partial A$ in the nonseparable model, we carry out the following steps:
\begin{enumerate}[itemsep = 0pt, topsep =1pt]
\item We first construct the indicator functions $\onev_{S_i}$ and $\onev_A$ on the corresponding graphs.
\item We then compute the empirical bandwidth $\omega(\onev_{S_i})$ and $\omega(\onev_{A})$ in a manner that takes care of numerical error: we first obtain the eigenvectors of the corresponding $\Lm$, then set $\omega(\onev_{S_i})$ and $\omega(\onev_{A})$ to be $\nu$ for which energy contained in the graph Fourier coefficients corresponding to eigenvalues $\lambda_j > \nu$ is at most $0.01\%$, i.e., 
\begin{align}
\omega(\onev_{S_i}) &= \min \big\{\nu \big| \sum_{j:\lambda_j > \nu} \left(\uv_j^T \onev_{S_i} \right)^2 \leq 10^{-4}\big\} \\
\omega(\onev_{A}) &= \min \big\{\nu \big| \sum_{j:\lambda_j > \nu} \left(\uv_j^T \onev_{A} \right)^2 \leq 10^{-4}\big\}.
\end{align}
\end{enumerate}
\begin{figure}[t]
\centerline{\includegraphics[width=0.9\linewidth]{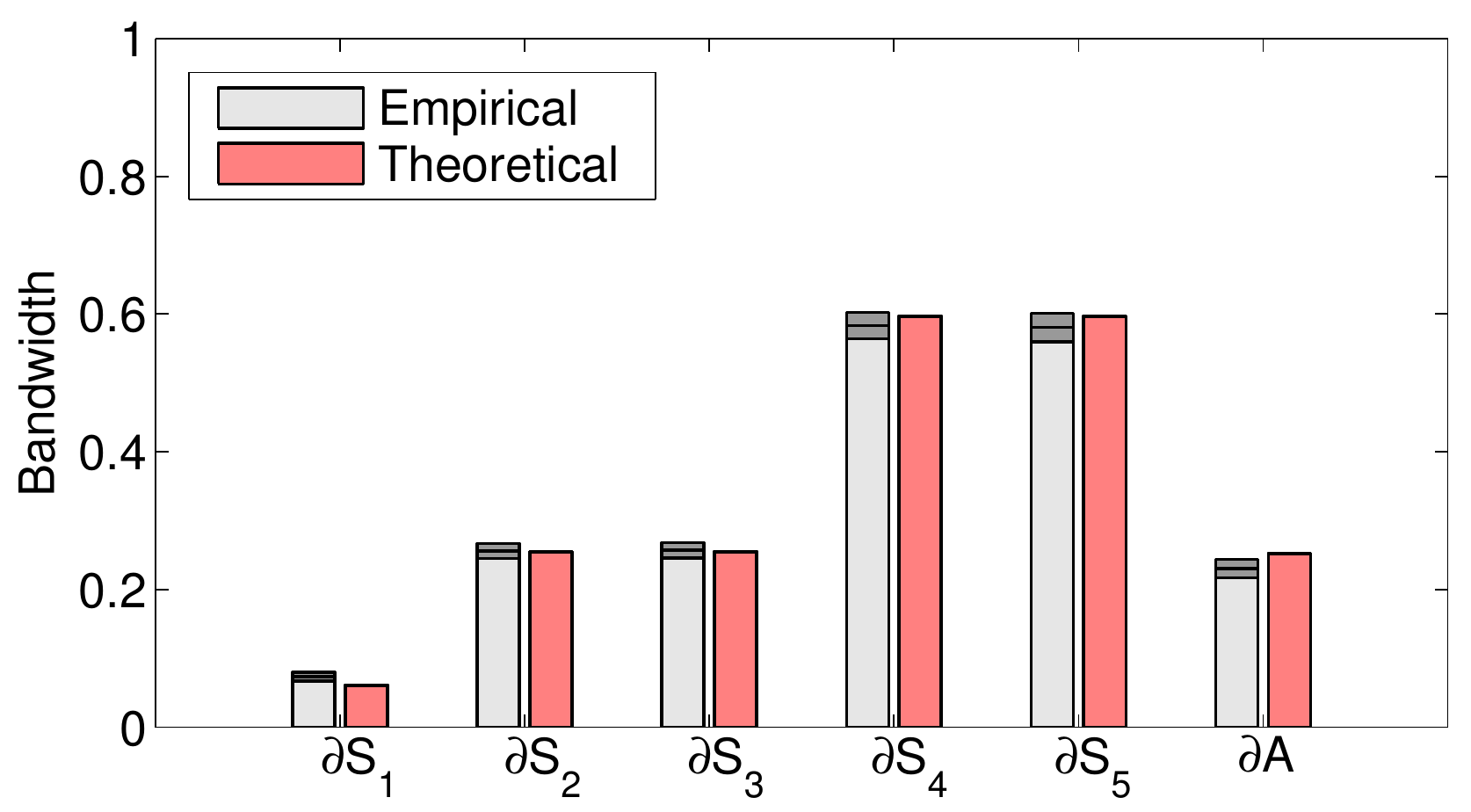}}
\caption{Convergence of empirical value of bandwidths $\omega(\onev_{S_i})$ and $\omega(\onev_{A})$ for different boundaries $\{\partial S_i\}$ and $\partial A$ on corresponding graphs. Dark shaded regions denote standard deviation over 100 experiments. Red bars indicate theoretical values.}
\label{fig:exp1}
\end{figure}
\begin{figure}[!]
\centerline{\includegraphics[width=0.9\linewidth]{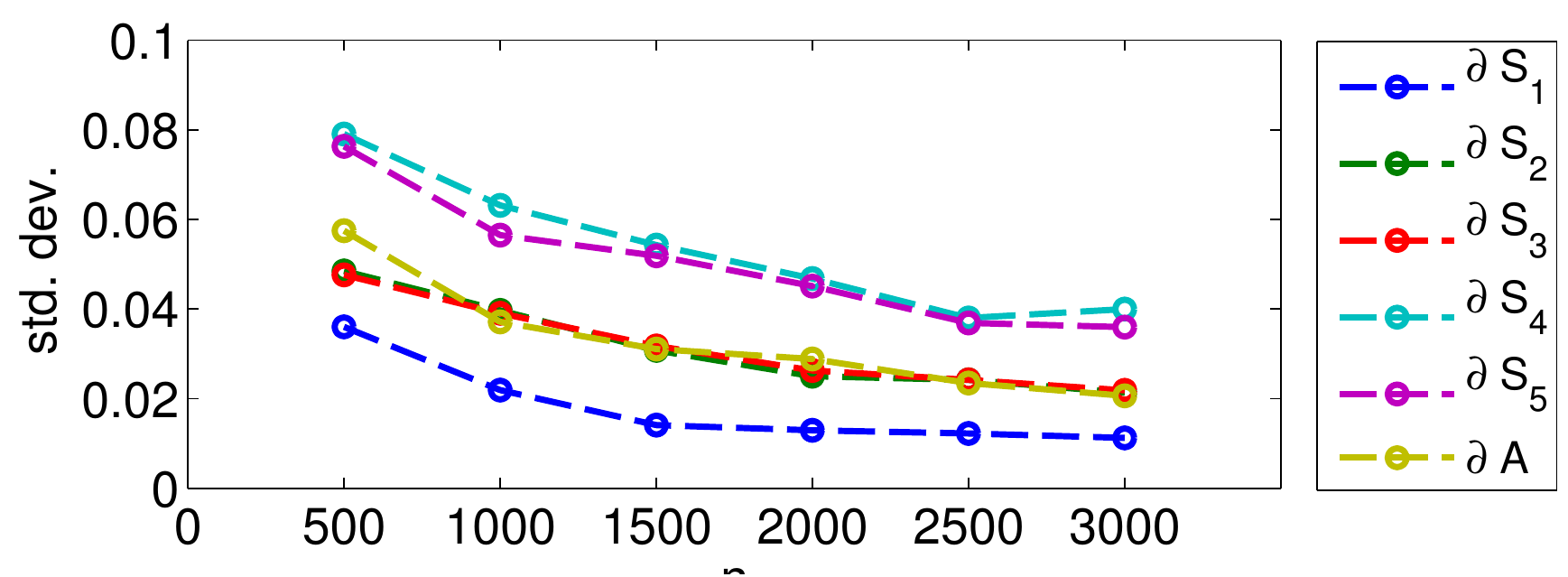}}
\caption{Standard deviation of $\omega(\onev_{S_i})$ and $\omega(\onev_{A})$ as a function of $n$.}
\label{fig:bw_vs_n}
\end{figure}
\begin{figure}[!]
\begin{center}
\begin{subfigure}{0.99\linewidth}
\centering
\includegraphics[width=0.9\linewidth]{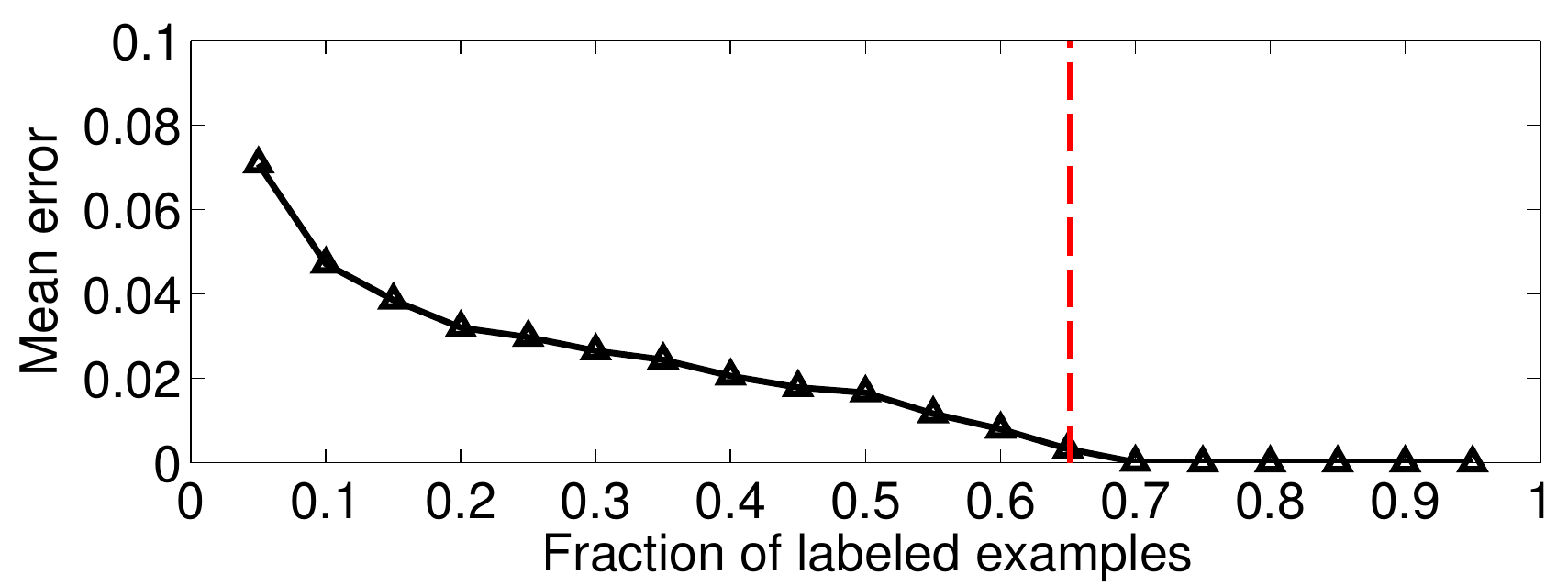}
\caption{}
\end{subfigure}
\\
\begin{subfigure}{0.99\linewidth}
\centering
\includegraphics[width=0.9\linewidth]{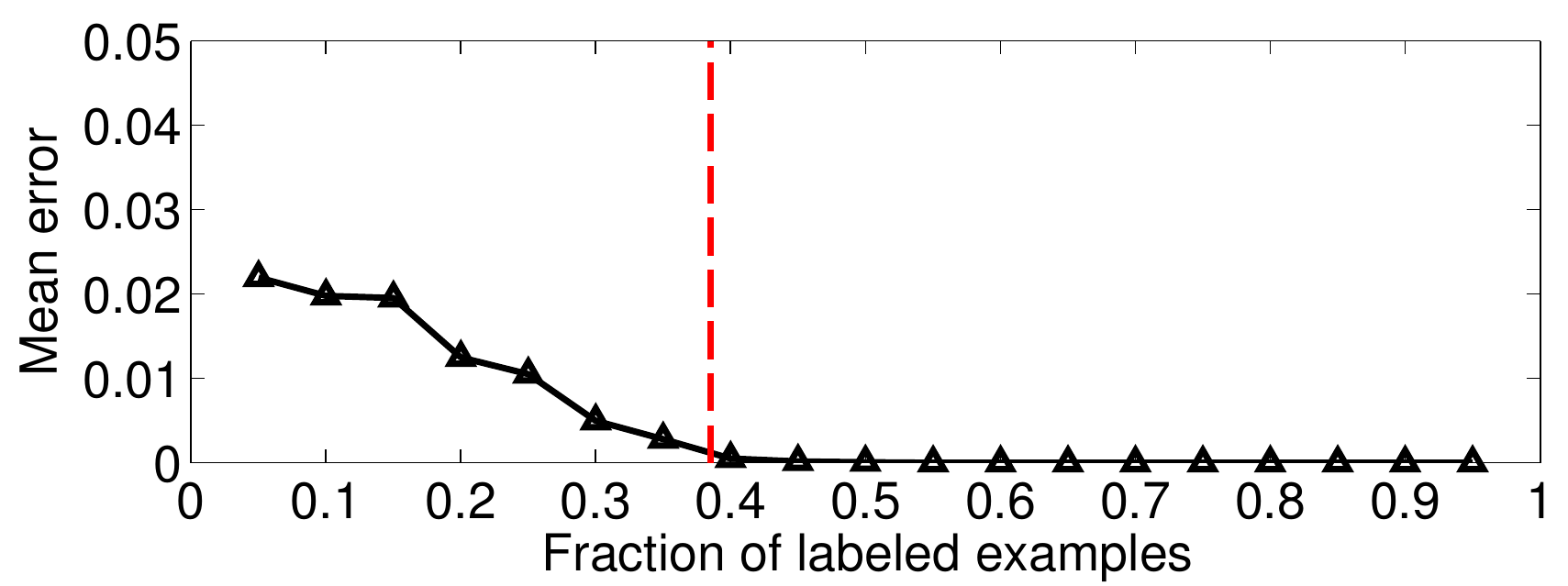}
\caption{}
\end{subfigure}
\end{center}
\caption{Mean reconstruction errors averaged over $100$ experiments for (a) $\onev_{S_3}$, and (b) $\onev_{A}$. Red-dashed lines indicate the limit values of label complexities predicted by our results, i.e., $P(\Xc_{S_3})$ and $P(\Xc_{A})$.}
\label{fig:exp3}
\end{figure}
The procedure above is repeated 100 times and the mean of $\omega(\onev_{S_i})$ and $\omega(\onev_{A})$ are compared with $\sup_{\sv \in \partial S_i} p(\sv)$ and $\sup_{\xv \in \partial A} p(\xv)$ respectively. The result is plotted in Figure~\ref{fig:exp1}. We observe that the empirical bandwidth is close to the theoretically predicted value and has a very low standard deviation. This supports our conjecture that stochastic convergence should hold for the bandwidth. To further justify this claim, we study the behavior of the standard deviation of $\omega(\onev_{S_i})$ and $\omega(\onev_{A})$ as a function of $n$ in Figure~\ref{fig:bw_vs_n}, where we observe a decreasing trend consistent with our result.
\par
For our second experiment, we validate the label complexity of sampling theory-based learning in Conjecture~\ref{conjecture:lc} by reconstructing the indicator function corresponding to $\partial S_3$ and $\partial A$ from a fraction of labeled examples on the corresponding graphs.
This is carried out as follows: For a given budget $l$, we find the set of points $L\subset \{1,2,\dots,n\}$ to label of size $|L| = l$, using pivoted column-wise Gaussian elimination on the eigenvector matrix $\Um$ of $\Lm$~\cite{anis_tsp_16}. This method ensures that the obtained labeled set guarantees perfect recovery for signals spanned by the first $l$ eigenvectors of $\Lm$~\cite{anis_tsp_16}. 
We then recover the indicator functions from these labeled sets by solving the least squares problem in~\eqref{eq:bl_interp} followed by thresholding. Note that $\theta$ is set to the cutoff frequency $\omega_c(L)$ of $L$, which is equal to the $l^\text{th}$ eigenvalue of $\Lm$. The mean reconstruction error is defined as
\begin{equation}
E_{\rm mean} = \frac{\text{No. of mismatches on unlabeled set}}{\text{Size of unlabeled set}}.
\end{equation}
We repeat the experiment $100$ times by generating different graphs and plot the averaged $E_{\rm mean}$ against the fraction of labeled examples.
The result is illustrated in Figure~\ref{fig:exp3}. We observe that the error goes to zero as the fraction of labeled points goes beyond the respective limit values stated in~\eqref{eq:sep_lc_limit} and~\eqref{eq:nonsep_lc_limit}. This reinforces the intuition that the bandwidth of class indicators and their label complexities are closely linked with the inherent geometry of the data.

\section{Discussions and future work}
\label{sec:end}

In this paper, we provided an interpretation of the graph sampling theoretic approach to semi-supervised learning. Our work analyzed the bandwidth of class indicator signals with respect to the Laplacian eigenvector basis and revealed its connection to the underlying geometry of the dataset. This connection is useful in justifying graph-based approaches for semi-supervised and unsupervised learning problems, and provides a geometrical interpretation of the smoothness assumptions imposed in the bandlimited reconstruction approach.
Specifically, our results have shown that an estimate of the bandwidth of class indicators converges to the supremum of the probability density on the class boundaries for the separable model, and on the overlap regions for the nonseparable model. This quantifies the connection between the assumptions of smoothness (in terms of bandlimitedness) and low density separation, since boundaries passing through regions of low data density result in lower bandwidth of the class indicator signals. We numerically validated these results through various experiments.

There are several directions in which our results can be extended. In this paper we only considered Gaussian-weighted graphs, an immediate extension would be to consider arbitrary kernel functions for computing graph weights, or density dependent edge-connections such as $k$-nearest neighbors. Another possibility is to consider data defined on a subset of the $d$-dimensional Euclidean space. 

Our analysis also sheds light on the label complexity of graph-based semi-supervised learning problems. We showed that perfect prediction from a few labeled examples using a graph-based bandlimited interpolation approach requires the same amount of labeling as one would need to completely encompass the boundary or region of ambiguity. This quantifies the connection between label complexity of a sampling theory-based approach with the underlying geometry of the problem. We believe that the main potential of graph-based methods will be apparent in situations where one can tolerate a certain amount of prediction error, in which case such approaches shall require fewer labeled data. We plan to investigate this as part of future work.

\bibliographystyle{ieeetr}
\bibliography{refs}

\appendices

\section{Proof of Lemma 5}
\label{app:proof_lemma_5}

The key ingredient required for evaluating the integrals in Lemma~\ref{lemma:flow} involves selecting a radius $R$ ($< \tau$) as a function of $\sigma$ that satisfies the following properties as $\sigma \rightarrow 0$: 
\begin{enumerate}[itemsep=0pt, topsep=3pt]
\item $R \rightarrow 0$,
\item $R/\sigma \rightarrow \infty$,
\item $R^2/\sigma \rightarrow 0$,
\item $\epsilon_R/\sigma \rightarrow 0$, where $\epsilon_R := \int_{\|\zv\|>R} K_{\sigma^2}(\zerov, \zv)d\zv$.
\end{enumerate}
A particular choice of $R$ is given by $R = \sqrt{d\sigma^2 \log{(1/\sigma^2)}}$. Note that $R \rightarrow 0$ as $\sigma \rightarrow 0$. Further,
\begin{align}
\frac{R}{\sigma} &= \sqrt{d \log{(1/\sigma^2)}}, \\
\frac{R^2}{\sigma} &= d\sigma \log{(1/\sigma^2)}.
\end{align}
Hence, $R/\sigma \rightarrow \infty$ and $R^2/\sigma \rightarrow 0$ as $\sigma \rightarrow 0$.
Additionally, substituting the expression for $R$ in the tail bound for the norm of a $d$-dimensional Gaussian vector gives us:
\begin{align}
\frac{\epsilon_R}{\sigma} &= \frac{1}{\sigma} \int_{\|\zv\|>R} K_{\sigma^2}(\zerov, \zv)d\zv  \nonumber \\
&\leq \frac{1}{\sigma}\left( \frac{\sigma^2 d}{R^2}\right)^{-d/2} e^{-\frac{R^2}{2\sigma^2} + \frac{d}{2}} \nonumber \\
&= \frac{1}{\sigma}\left( e \sigma^2 \log(1/\sigma^2)\right)^{d/2}.
\end{align}
Therefore, for $d > 1$, $\epsilon_R/\sigma \rightarrow 0$ as $\sigma \rightarrow 0$. Further, it is easy to ensure $R < \tau$ for the regime of $\sigma$ in our proofs.

We now consider the proof of equation \eqref{eq:flow_one}, let 
\begin{equation}
I := \frac{1}{\sigma} \int_{S_1} \int_{S_2} K_{\sigma^2}(\xv_1,\xv_2) p^\alpha(\xv_1) p^\beta(\xv_2) d\xv_1 d\xv_2.
\end{equation}
Further, let $[S_1]_R$ indicate a tubular region of thickness $R$ adjacent to the boundary $\partial S$ in $S_1$, i.e., the set of points in $S_1$ at a distance $\leq R$ from the boundary. Then, we have
\begin{align}
&I = \underbrace{\frac{1}{\sigma} \int_{[S_1]_R} p^\alpha(\xv_1) \int_{S_2} K_{\sigma^2}(\xv_1,\xv_2) p^\beta(\xv_2) d\xv_2 \; d\xv_1}_{I_1} \nonumber \\ 
&\quad+ \underbrace{\frac{1}{\sigma} \int_{[S_1]_R^c} p^\alpha(\xv_1) \int_{S_2} K_{\sigma^2}(\xv_1,\xv_2) p^\beta(\xv_2) d\xv_2 \; d\xv_1}_{E_1}.
\end{align}
$E_1$ is the error associated with approximating $I$ by $I_1$ and exhibits the following behavior: 

\begin{lemma}
$\lim_{\sigma \rightarrow 0} E_1 = 0$.
\end{lemma}
\begin{proof}
Note that
\begin{align}
E_1 
&\leq \frac{1}{\sigma} \left( p_{\rm max} \right)^\beta \int_{[S_1]_R^c} p^\alpha(\xv_1) \left( \int_{S_2} K_{\sigma^2}(\xv_1,\xv_2) d\xv_2 \right) d\xv_1 \nonumber \\
&\leq \frac{1}{\sigma} \left( p_{\rm max} \right)^\beta \int_{[S_1]_R^c} p^\alpha(\xv_1) \left( \int_{\|\zv\| > R} K_{\sigma^2}(\zerov,\zv) d\zv \right) d\xv_1 \nonumber \\
&= \frac{\epsilon_R}{\sigma} \left( p_{\rm max} \right)^\beta \int_{[S_1]_R^c} p^\alpha(\xv_1) d\xv_1 \nonumber \\
&\leq \frac{\epsilon_R}{\sigma} \left( p_{\rm max} \right)^{\alpha + \beta}.
\end{align}
Using $\lim_{\sigma \rightarrow \infty} \epsilon_R/\sigma = 0$, we get the desired result.
\end{proof}

In order to analyze $I_1$, we need to define certain geometrical constructions (illustrated in Figure~\ref{fig:constructions}) as follows: 
\begin{definition}
\label{def:constructions}
\begin{enumerate}
\item
For each $\xv_1 \in [S_1]_R$, we define a transformation of coordinates as:
\begin{equation}
\xv_1 = \sv_1 + r_1 \nv(\sv_1),
\end{equation}
where $\sv_1$ is the foot of the perpendicular dropped from $\xv_1$ onto $\partial S$, $r_1$ is the distance between $\sv_1$ and $\xv_1$, and $\nv(\sv_1)$ is the surface normal at $\sv_1$ (towards the direction of $\xv_1$). Since the minimum radius of curvature of $\partial S$ is $\tau$ and $R < \tau$, this mapping is injective.
\item For each $\sv_1 \in \partial S$, let $H_{\sv_1}^+$ denote the half-space created by the plane tangent on $\sv_1$ and on the side of $S_2$. Similarly, let $H_{\sv_1}^-$ denote the half-space on the side of $S_1$, that is, $H_{\sv_1}^- = \mathbb{R}^d \setminus H_{\sv_1}^+ $.
\item Let $W_{\sv_1}^+(x)$ denote an infinite slab of thickness $x$ tangent to $\partial S$ at $\sv_1$ and towards the side of $S_2$. Let $W_{\sv_1}^-(y)$ denote a similar slab of thickness $y$ on the side of $S_1$.
\item Finally, for any $\xv$, let $B(\xv,R)$ denote the Euclidean ball of radius $R$ centered at $\xv$.
\end{enumerate}
\end{definition}
\begin{figure}[t]
\centering
\includegraphics[width=0.4\linewidth]{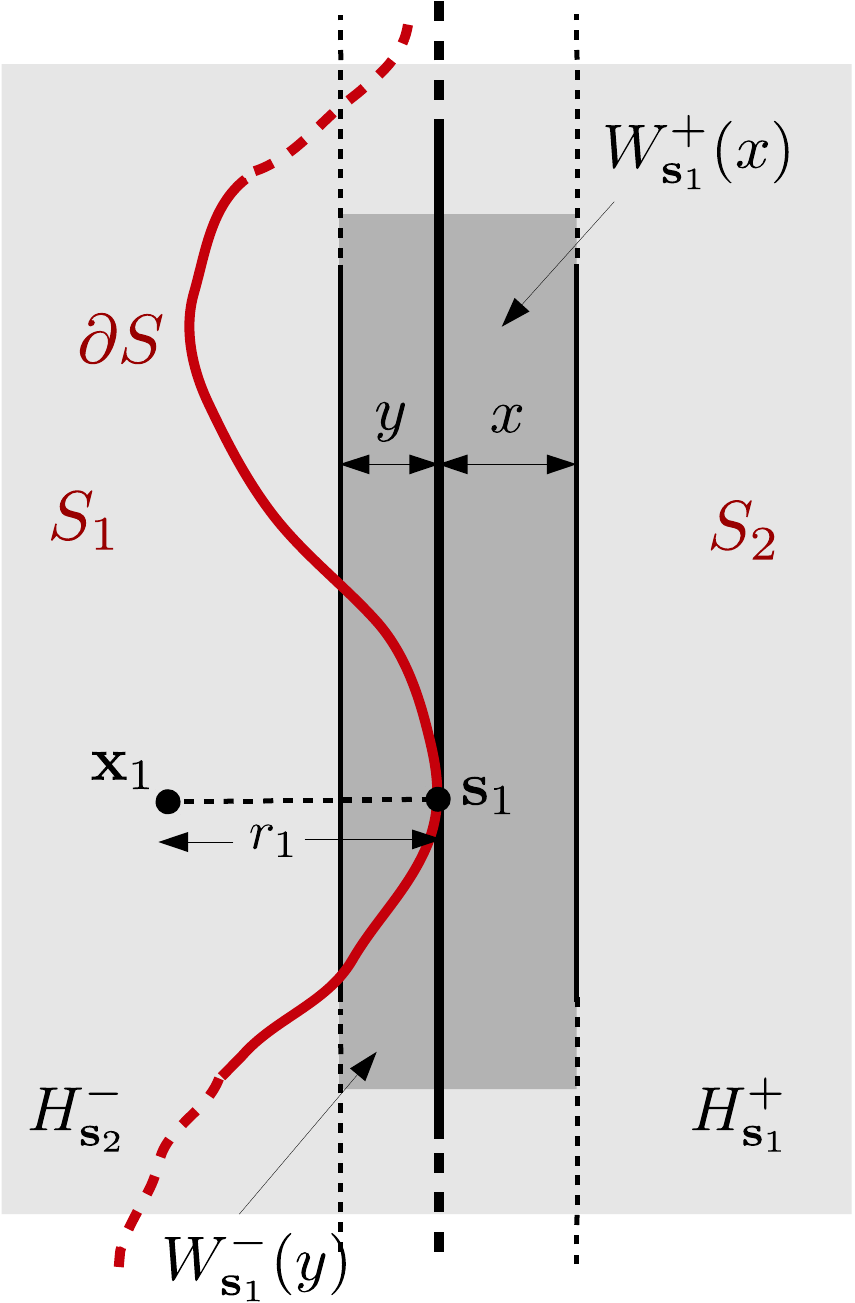}
\caption{Geometrical constructions in Definition~\ref{def:constructions}.}
\label{fig:constructions}
\end{figure}
We now consider $I_1$, the main idea here is to approximate the integral over $S_2$ by an integral over the half-space $H^+_{\sv_1}$. Hence, we have:
\begin{align}
&I_1 = \underbrace{\frac{1}{\sigma} \int_{[S_1]_R} p^\alpha(\xv_1) \int_{H^+_{\sv_1}} K_{\sigma^2}(\xv_1,\xv_2) p^\beta(\xv_2) d\xv_2 d\xv_1}_{I_2} \nonumber \\
&+ \underbrace{ \frac{1}{\sigma} \int_{[S_1]_R} p^\alpha(\xv_1) \int_{S_2 - H^+_{\sv_1}} K_{\sigma^2}(\xv_1,\xv_2) p^\beta(\xv_2) d\xv_2 d\xv_1}_{E_2},
\end{align}
where $E_2$ is the error associated with the approximation. Therefore, we have
\begin{equation}
I = I_2 + E_2 + E_1.
\end{equation}
We now show that as $\sigma \rightarrow 0$, $I_2 \rightarrow \frac{1}{\sqrt{2\pi}}\int_{\partial S}p^{\alpha+\beta}(\sv)d\sv$, and $E_2 \rightarrow 0$.
\begin{lemma}
$\lim_{\sigma \rightarrow 0} I_2 = \frac{1}{\sqrt{2\pi}} \int_{\partial S} p^{\alpha + \beta}(\sv) d\sv $.
\end{lemma}
\begin{proof}
Using the change of coordinates $\xv_1 = \sv_1 + r_1 \nv(\sv_1)$, we have
\begin{align}
I_2 &= \frac{1}{\sigma} \int_{\partial S} \int_0^R p^\alpha(\sv_1 + r_1 \nv(\sv_1)) \nonumber \\
&\qquad \qquad \left( \int_{H^+_{\sv_1}} K_{\sigma^2}(\sv_1 + r_1 \nv(\sv_1),\xv_2) p^\beta(\xv_2) d\xv_2 \right) \nonumber \\
&\qquad \qquad \qquad |\text{det}J(\sv_1,r_1)| d\sv_1 dr_1,
\end{align}
where $J(\sv_1,r_1)$ denotes the Jacobian of the transformation. Now, an arc $\widehat{PQ}$ of length $ds$ at a distance $r_1$ away from $\partial S$ gets mapped to an arc $\widehat{P'Q'}$ on $\partial S$ whose length lies in the interval $[ds(1-\frac{r_1}{\tau}),ds(1+\frac{r_1}{\tau})]$. Therefore, for all points within $[S_1]_R$, we have
\begin{equation}
\left( 1-\frac{R}{\tau} \right)^{d-1} \leq |\text{det}J(\sv_1,r_1)| \leq \left( 1+\frac{R}{\tau} \right)^{d-1}.
\end{equation}
Further, since $p(\xv)$ is Lipschitz continuous with constant $L_p$, $p^\alpha(\xv)$ is also Lipschitz continuous with constant $L_{p,\alpha}$. Therefore, for any $\xv_1 \in [S_1]_R$, we have $p^\alpha(\xv_1) = p^\alpha(\sv_1) + L_{p,\alpha}R$. This leads to the following simplification for $I_2$:
\begin{align}
I_2 &= \left( 1 + O(R^{d-1}) \right) \int_{\partial S}  p^\alpha(\sv_1) I_3(\sv_1) d\sv_1 \nonumber \\ 
&\qquad \qquad + O(R^d) \int_{\partial S} I_3(\sv_1) d\sv_1,
\label{eq:I_2}
\end{align}
where we defined
\begin{equation}
I_3(\sv_1) := \frac{1}{\sigma} \int_0^R \int_{H_{\sv_1}^+} K_{\sigma^2}(\sv_1 + r_1 \nv(\sv_1),\xv_2) p^\beta(\xv_2) d\xv_2  dr_1.
\end{equation}
Note that every $\xv_2 \in H_{\sv_1}^+$ can be written as $\sv_2 + r_2 \nv(\sv_2)$, where $\nv(\sv_2) = -\nv(\sv_1)$. Hence, we get
\begin{align}
I_3(\sv_1) &= \int_{\mathbb{R}^{d-1}}  \frac{1}{(2\pi\sigma^2)^{\frac{d-1}{2}}} e^{-\frac{\|\sv_1 - \sv_2\|^2}{2\sigma^2}} p^\beta(\sv_2 - r_2\nv(\sv_1)) d\sv_2 \nonumber \\
&\qquad\quad \times \frac{1}{\sigma} \int_0^R \int_0^\infty \frac{1}{\sqrt{2\pi\sigma^2}} e^{-\frac{(r_1 + r_2)^2}{2\sigma^2}}   dr_1 dr_2 \nonumber \\ 
&= \left(\int_{\mathbb{R}^{d-1}}  \frac{1}{(2\pi\sigma^2)^{\frac{d-1}{2}}} e^{-\frac{\|\sv_1 - \sv_2\|^2}{2\sigma^2}} p^\beta(\sv_2) d\sv_2 + O(R) \right) \; \nonumber \\ 
&\qquad\quad \times \frac{1}{\sigma} \int_0^R \int_0^\infty \frac{1}{\sqrt{2\pi\sigma^2}} e^{-\frac{(r_1 + r_2)^2}{2\sigma^2}}   dr_1 dr_2 \nonumber \\
&= \left(p^\beta(\sv_1) + O(\sigma^2) + O(R) \right) \; \times \nonumber \\
&\qquad\quad \frac{1}{\sigma} \int_0^R \int_0^\infty \frac{1}{\sqrt{2\pi\sigma^2}} e^{-\frac{(r_1 + r_2)^2}{2\sigma^2}}   dr_1 dr_2,
\label{eq:I_3_evaluation}
\end{align}
where we used Lipschitz continuity of $p^\beta(\xv)$ in the second equality and applied Lemma~\ref{lem:kernelone} to arrive at the last step. Further, using the definition of the $Q$-function and integration by parts, we note that
\begin{align}
&\frac{1}{\sigma} \int_0^R \int_0^\infty \frac{1}{\sqrt{2\pi\sigma^2}} e^{-\frac{(r_1 + r_2)^2}{2\sigma^2}}   dr_1 dr_2 \nonumber \\
&\quad = \int_0^{R/\sigma} \int_{0}^\infty \frac{1}{\sqrt{2\pi}} e^{-\frac{(x + y)^2}{2}} dx  dy  \nonumber \\
&\quad = \int_0^{R/\sigma} Q(y)  dy  \nonumber \\
&\quad = y  Q(y) \bigg|_0^{R/\sigma} - \int_0^{R/\sigma}  Q'(y) dy \nonumber \\
&\quad = \frac{R}{\sigma} Q\left(\frac{R}{\sigma}\right)  + \frac{1}{\sqrt{2\pi}} \left( 1 - e^{-R^2/2\sigma^2} \right). \nonumber
\end{align}
Therefore,
\begin{align}
I_3(\sv_1) &= \left(p^\beta(\sv_1) + O(\sigma^2) + O(R) \right) \times \nonumber \\
&\qquad \left( \frac{R}{\sigma} Q\left(\frac{R}{\sigma}\right) + \frac{1}{\sqrt{2\pi}} \left( 1 - e^{-R^2/2\sigma^2} \right) \right).
\label{eq:I_3}
\end{align}
Combining~\eqref{eq:I_2} and~\eqref{eq:I_3} and using the fact that $R/\sigma \rightarrow \infty$ as $\sigma \rightarrow 0$ (from the definition of $R$), we get
\begin{equation}
\lim_{\sigma \rightarrow 0} I_2 = \frac{1}{\sqrt{2\pi}} \int_{\partial S} p^{\alpha + \beta}(\sv) d\sv,
\end{equation}
which concludes the proof.
\end{proof}

We now consider the error term $E_2$ and prove the following result:
\begin{lemma}
$\lim_{\sigma \rightarrow 0} E_2 = 0$.
\end{lemma}
\begin{figure*}
\centering
\begin{subfigure}{0.48\textwidth}
\centering
\includegraphics[width=0.8\linewidth]{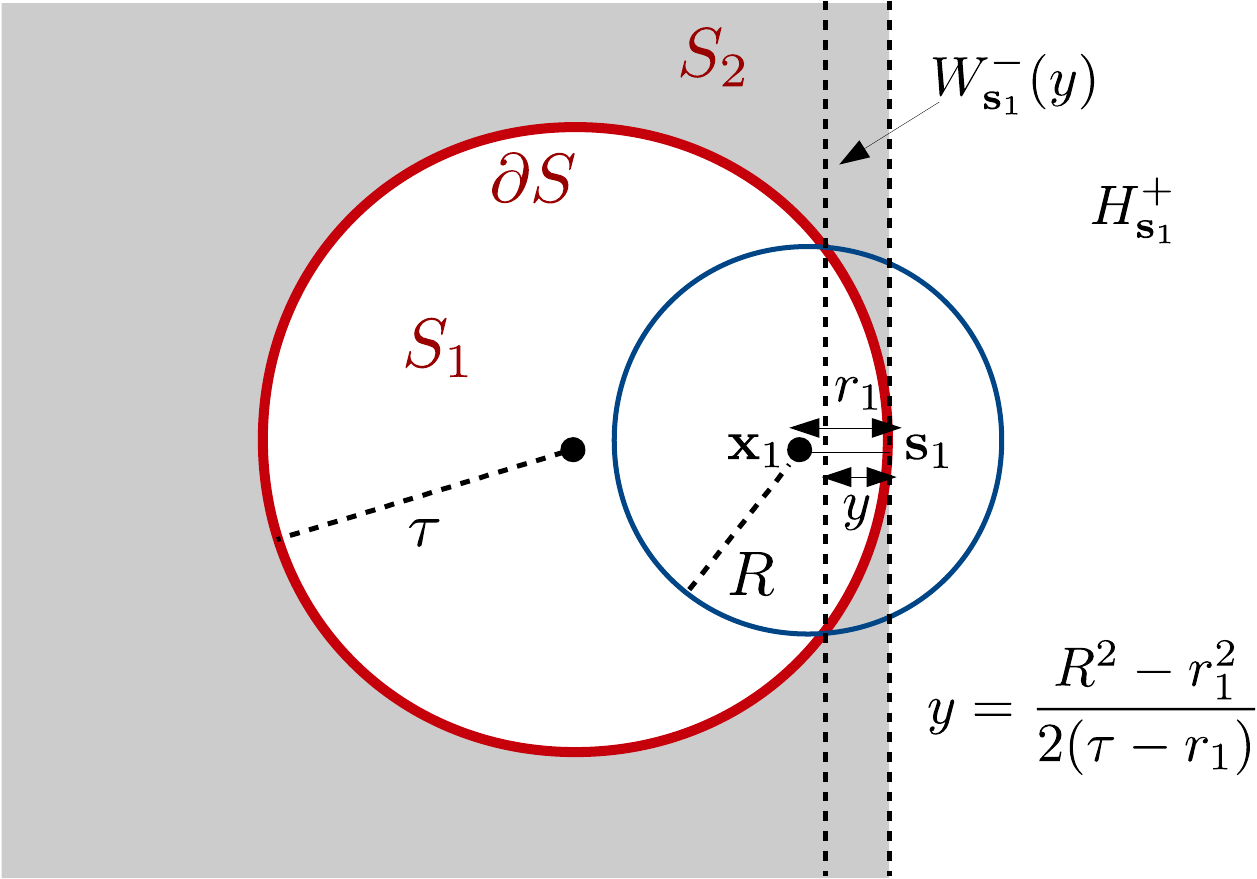}
\caption{}
\label{fig:upper_bound}
\end{subfigure}
\quad
\begin{subfigure}{0.48\textwidth}
\centering
\includegraphics[width=0.8\linewidth]{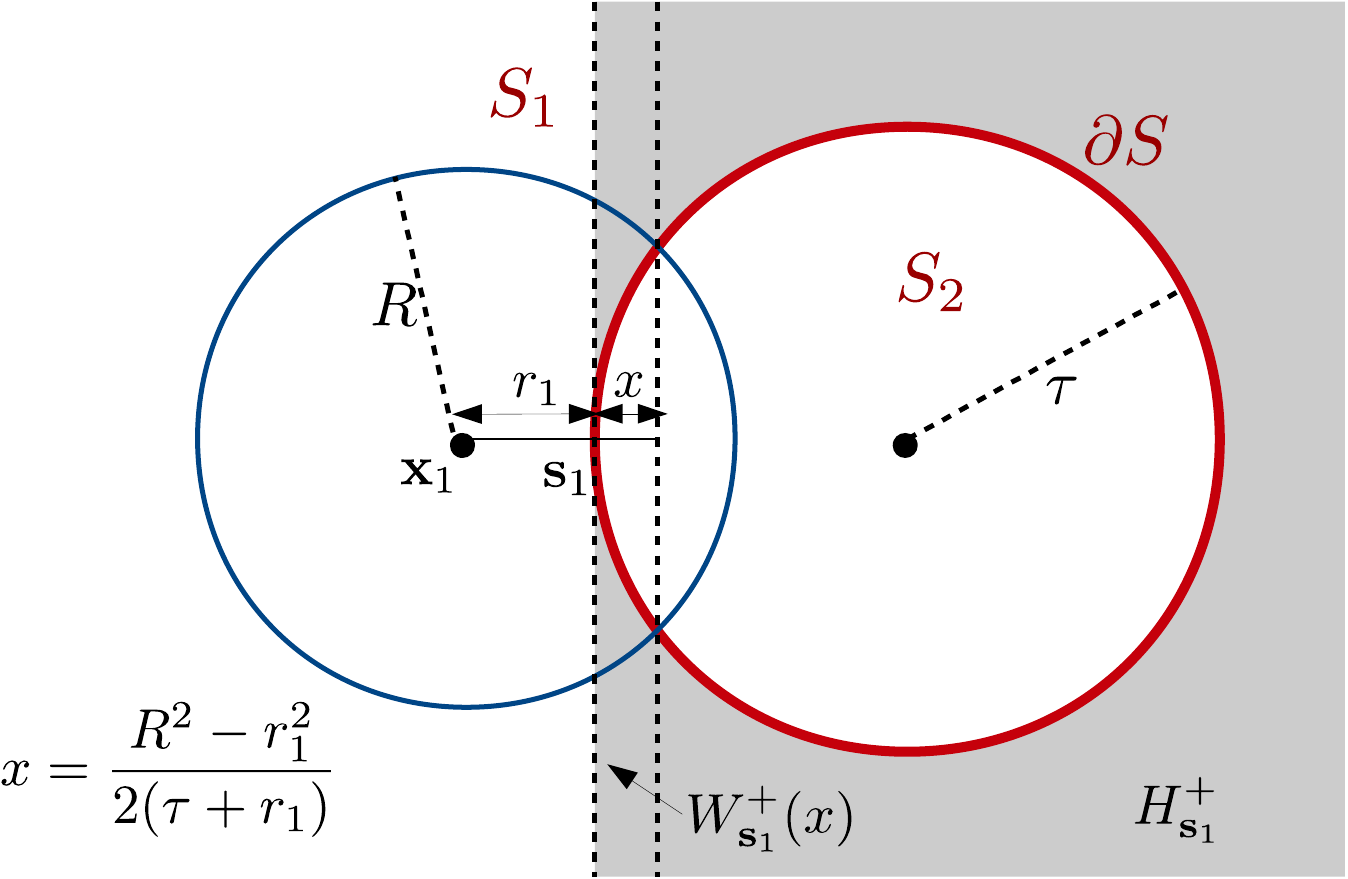}
\caption{}
\label{fig:lower_bound}
\end{subfigure}
\caption{Worst-case scenarios for the boundary $\partial S$ when (a) $S_1$ is a ball of radius $\tau$, (b) $S_2$ is a ball of radius $\tau$.}
\end{figure*}
\begin{proof}
Let us first rewrite $E_2$ as follows:
\begin{equation}
E_2 = \frac{1}{\sigma} \int_{[S_1]_R} p^\alpha(\xv_1) I_4(\xv_1) d\xv_1,
\label{eq:E_2}
\end{equation}
where we defined
\begin{equation}
I_4(\xv_1) := \int_{S_2 - H^+_{\sv_1}} K_{\sigma^2}(\xv_1,\xv_2) p^\beta(\xv_2) d\xv_2.
\end{equation}
The key idea is to lower and upper bound $I_4(\xv_1)$ for all $\xv_1$ using worst case scenarios and evaluate the limits of the bounds. Note that $I_4(\xv_1)$ is largest in magnitude when $S_1$ or $S_2$ is a sphere of radius $\tau$, as illustrated in Figures~\ref{fig:upper_bound} and~\ref{fig:lower_bound}. We now make certain geometrical observations. For any $\xv_1 = \sv_1 + r_1 \nv(\sv_1) \in [S_1]_R$, we observe from Figure~\ref{fig:lower_bound} that
\begin{align}
I_4(\xv_1) &\leq \int_{W_{\sv_1}^-\left( \frac{R^2 - r_1^2}{2(\tau - r_1)} \right)} K_{\sigma^2}(\xv_1,\xv_2) p^\beta(\xv_2) d\xv_2 \nonumber \\
&\qquad \quad + \int_{B(\xv_1,R)^c} K_{\sigma^2}(\xv_1,\xv_2) p^\beta(\xv_2) d\xv_2 \nonumber \\
&\leq \int_{W_{\sv_1}^-(R')} K_{\sigma^2}(\xv_1,\xv_2) p^\beta(\xv_2) d\xv_2  + p_{\rm max}^\beta \epsilon_R.
\end{align}
where $R' = \frac{R^2}{2(\tau - R)}$. Similarly, from Figure~\ref{fig:upper_bound}, we observe that
\begin{align}
I_4(\xv_1) &\geq - \bigg[ \int_{W_{\sv_1}^+\left( \frac{R^2 - r_1^2}{2(\tau + r_1)} \right)} K_{\sigma^2}(\xv_1,\xv_2) p^\beta(\xv_2) d\xv_2 \nonumber \\
&\qquad \quad + \int_{B(\xv_1,R)^c} K_{\sigma^2}(\xv_1,\xv_2) p^\beta(\xv_2) d\xv_2 \bigg] \nonumber \\
&\geq - \bigg[ \int_{W_{\sv_1}^+(R')} K_{\sigma^2}(\xv_1,\xv_2) p^\beta(\xv_2) d\xv_2  + p_{\rm max}^\beta \epsilon_R \bigg].
\end{align}
Substituting these in \eqref{eq:E_2} and using a simplification similar to that of $I_2$ in \eqref{eq:I_2}, we get
\begin{align}
E_2 &\leq \left( 1 + O(R^{d-1}) \right) \int_{\partial S}  p^\alpha(\sv_1) I_5^-(\sv_1) d\sv_1 \nonumber \\ 
&\qquad + O(R^d) \int_{\partial S} I_5^-(\sv_1) d\sv_1 + \frac{\epsilon_R}{\sigma} p_{\rm max}^{\alpha + \beta}, \\
E_2 &\geq -\left( 1 + O(R^{d-1}) \right) \int_{\partial S}  p^\alpha(\sv_1) I_5^+(\sv_1) d\sv_1 \nonumber \\ 
&\qquad - O(R^d) \int_{\partial S} I_5^+(\sv_1) d\sv_1 - \frac{\epsilon_R}{\sigma} p_{\rm max}^{\alpha + \beta},
\end{align}
where we defined
\begin{align}
&I_5^-(\sv_1) := \frac{1}{\sigma} \int_0^R \int_{W_{\sv_1}^-(R')} K_{\sigma^2}(\sv_1 + r_1 \nv(\sv_1),\xv_2) \nonumber \\
&\qquad\qquad\qquad\qquad p^\beta(\xv_2) d\xv_2  dr_1, \\
&I_5^+(\sv_1) := \frac{1}{\sigma} \int_0^R \int_{W_{\sv_1}^+(R')} K_{\sigma^2}(\sv_1 + r_1 \nv(\sv_1),\xv_2) \nonumber \\
&\qquad\qquad\qquad\qquad p^\beta(\xv_2) d\xv_2  dr_1.
\end{align}
Similar to the evaluation of $I_3(\sv_1)$ in~\eqref{eq:I_3_evaluation}, we have
\begin{align}
I_5^+(\sv_1) &= \left(p^\beta(\sv_1) + O(\sigma^2) + O(R) \right) \; \times \nonumber \\
&\qquad\quad \frac{1}{\sigma} \int_0^R \int_0^{R'} \frac{1}{\sqrt{2\pi\sigma^2}} e^{-\frac{(r_1 + r_2)^2}{2\sigma^2}}   dr_1 dr_2, \\
I_5^-(\sv_1) &= \left(p^\beta(\sv_1) + O(\sigma^2) + O(R) \right) \; \times \nonumber \\
&\qquad\quad \frac{1}{\sigma} \int_0^R \int_0^{R'} \frac{1}{\sqrt{2\pi\sigma^2}} e^{-\frac{(r_1 - r_2)^2}{2\sigma^2}}   dr_1 dr_2.
\end{align}
We now evaluate the two 1-D integrals as follows:
\begin{align}
&\frac{1}{\sigma} \int_0^R \int_0^{R'} \frac{1}{\sqrt{2\pi\sigma^2}} e^{-\frac{(r_1 + r_2)^2}{2\sigma^2}}   dr_1 dr_2 \nonumber \\
&\quad = \int_0^{R/\sigma} \int_{0}^{R'/\sigma} \frac{1}{\sqrt{2\pi}} e^{-\frac{(x + y)^2}{2}} dx  dy  \nonumber \\
&\quad = \int_0^{R/\sigma} \left( Q(y) - Q\left(y + \frac{R'}{\sigma}\right) \right) dy  \nonumber \\
&\quad = \int_0^{R/\sigma} Q(y) dy +  \int_0^{R'/\sigma} Q(y)dy - \int_0^{\frac{R+R'}{\sigma}}  Q(y) dy  \nonumber \\
&\quad = \frac{R}{\sigma} Q\left(\frac{R}{\sigma}\right) + \frac{R}{\sigma} Q\left(\frac{R'}{\sigma}\right) - \frac{R+R'}{\sigma} Q\left(\frac{R+R'}{\sigma}\right) \nonumber \\ 
&\qquad\qquad \frac{1}{\sqrt{2\pi}} \left( 1 - e^{-\frac{R^2}{2\sigma^2}} - e^{-\frac{R'^2}{2\sigma^2}} + e^{-\frac{(R+R')^2}{2\sigma^2}} \right) \nonumber.
\end{align}
Similarly,
\begin{align}
&\frac{1}{\sigma} \int_0^R \int_0^{R'} \frac{1}{\sqrt{2\pi\sigma^2}} e^{-\frac{(r_1 - r_2)^2}{2\sigma^2}}   dr_1 dr_2 \nonumber \\
&\quad = \int_0^{R/\sigma} \int_{0}^{R'/\sigma} \frac{1}{\sqrt{2\pi}} e^{-\frac{(x - y)^2}{2}} dx  dy  \nonumber \\
&\quad = \int_0^{R/\sigma} \left( Q\left(y - \frac{R'}{\sigma}\right) - Q(y) \right) dy  \nonumber \\
&\quad = \int_{-R'/\sigma}^0 Q(y)dy + \int_0^{\frac{R-R'}{\sigma}}  Q(y) dy - \int_0^{R/\sigma} Q(y) dy   \nonumber \\
&\quad = \frac{R'}{\sigma} Q\left(-\frac{R'}{\sigma}\right) + \frac{R-R'}{\sigma} Q\left(\frac{R-R'}{\sigma}\right) - \frac{R}{\sigma} Q\left(\frac{R}{\sigma}\right) \nonumber \\ 
&\qquad\qquad \frac{1}{\sqrt{2\pi}} \left( e^{-\frac{R'^2}{2\sigma^2}} - 1 + e^{-\frac{(R + R')^2}{2\sigma^2}} - e^{-\frac{R^2}{2\sigma^2}} \right). \nonumber
\end{align}
Noting that as $\sigma \rightarrow 0$, $R/\sigma \rightarrow \infty$ and $R'/\sigma \rightarrow 0$, we conclude that $\lim_{\sigma \rightarrow 0} E_2 = 0$.
\end{proof}

The proof of~\eqref{eq:flow_two} proceeds in a similar fashion by approximating the inner integral using hyperplanes. Specifically, similar to the proof of~\eqref{eq:flow_one}, we can show that the integral on the left hand side can be written as $I + E$, where
\begin{align}
I &:= \frac{1}{\sigma} \int_{[S_1]_R} \int_{H^-_{\sv_1}} \Big[ K_{a\sigma^2}(\xv_1,\xv_2) p^\alpha(\xv_1) p^\beta(\xv_2) \nonumber \\
&\quad \quad \quad \quad\quad - K_{b\sigma^2}(\xv_1,\xv_2) p^{\alpha'}(\xv_1) p^{\beta'}(\xv_2) \Big] d\xv_1 d\xv_2,
\end{align}
and $E$ is the residual associated with the approximation that can be shown to go to zero as $\sigma \rightarrow 0$ (we skip this proof since it is quite similar to the analysis for~\eqref{eq:flow_one}). In order to evaluate $I$, we perform a change of coordinates $\xv_1 = \sv_1 + r_1 \nv(\sv_1)$ as before to obtain
\begin{align}
I &= \frac{1}{\sigma} \int_{\partial S} \int_0^R \Bigg[  p^\alpha(\sv_1 + r_1 \nv(\sv_1)) \nonumber \\
&\qquad \qquad \bigg( \int_{H^-_{\sv_1}} K_{a\sigma^2}(\sv_1 + r_1 \nv(\sv_1),\xv_2) p^\beta(\xv_2) d\xv_2 \bigg) \nonumber \\
&\qquad\qquad\qquad -p^{\alpha'}(\sv_1 + r_1 \nv(\sv_1)) \nonumber \\
&\qquad \qquad \bigg( \int_{H^-_{\sv_1}} K_{b\sigma^2}(\sv_1 + r_1 \nv(\sv_1),\xv_2) p^{\beta'}(\xv_2) d\xv_2 \bigg) \Bigg] \nonumber \\
&\qquad \qquad \qquad |\text{det}J(\sv_1,r_1)| d\sv_1 dr_1 \nonumber \\
&= \int_{\partial S} p^\alpha(\sv_1) I_\beta(\sv_1) d\sv_1 - \int_{\partial S} p^{\alpha'}(\sv_1) I_{\beta'}(\sv_1) d\sv_1 +  O\left( R^d \right),
\end{align}
where we defined
\begin{align}
I_\beta(\sv_1) &:= \frac{1}{\sigma} \int_0^R \int_{H_{\sv_1}^-} K_{a\sigma^2}(\sv_1 + r_1 \nv(\sv_1),\xv_2)  p^\beta(\xv_2) d\xv_2  dr_1, \nonumber \\
I_{\beta'}(\sv_1) &:= \frac{1}{\sigma} \int_0^R \int_{H_{\sv_1}^-} K_{b\sigma^2}(\sv_1 + r_1 \nv(\sv_1),\xv_2) p^{\beta'}(\xv_2) d\xv_2  dr_1. \nonumber
\end{align}
By using a change of coordinates for $\xv_2$ similar to the steps in~\eqref{eq:I_3_evaluation}, we obtain
\begin{align}
I_\beta(\sv_1) &= \left(p^\beta(\sv_1) + O(\sigma^2) + O(R) \right) \; \times \nonumber \\
&\qquad\quad \frac{1}{\sigma} \int_0^R \int_0^{\infty} \frac{1}{\sqrt{2\pi a\sigma^2}} e^{-\frac{(r_1 - r_2)^2}{2a\sigma^2}} dr_1 dr_2, \\ 
I_{\beta'}(\sv_1) &= \left(p^{\beta'}(\sv_1) + O(\sigma^2) + O(R) \right) \; \times \nonumber \\
&\qquad\quad \frac{1}{\sigma} \int_0^R \int_0^{\infty} \frac{1}{\sqrt{2\pi b\sigma^2}} e^{-\frac{(r_1 - r_2)^2}{2b\sigma^2}} dr_1 dr_2.
\end{align}
The 1-D integrals can be evaluated as follows:
\begin{align}
&\frac{1}{\sigma} \int_0^R \int_0^\infty \frac{1}{\sqrt{2\pi a\sigma^2}} e^{-\frac{(r_1 - r_2)^2}{2a\sigma^2}}   dr_1 dr_2 \nonumber \\
&\quad = \sqrt{a} \int_0^{R/\sqrt{a}\sigma} \int_{0}^\infty \frac{1}{\sqrt{2\pi}} e^{-\frac{(x - y)^2}{2}} dx  dy  \nonumber \\
&\quad = \sqrt{a} \int_0^{R/\sqrt{a}\sigma} Q(-y)  dy  \nonumber \\
&\quad = \sqrt{a} \int_0^{R/\sqrt{a}\sigma} (1 - Q(y)) dy  \nonumber \\
&\quad = \frac{R}{\sigma} - \frac{R}{\sigma} Q\left(\frac{R}{\sqrt{a}\sigma}\right)  - \frac{\sqrt{a}}{\sqrt{2\pi}} \left( 1 - e^{-R^2/2a\sigma^2} \right), \nonumber
\end{align}
\begin{align}
&\frac{1}{\sigma} \int_0^R \int_0^\infty \frac{1}{\sqrt{2\pi b\sigma^2}} e^{-\frac{(r_1 - r_2)^2}{2b\sigma^2}}   dr_1 dr_2 \nonumber \\
&\quad = \frac{R}{\sigma} - \frac{R}{\sigma} Q\left(\frac{R}{\sqrt{b}\sigma}\right)  - \frac{\sqrt{b}}{\sqrt{2\pi}} \left( 1 - e^{-R^2/2b\sigma^2} \right). \nonumber
\end{align}
Using the fact that $\alpha +\beta = \alpha' + \beta' = \gamma$, and taking the limit $\sigma \rightarrow 0$ after putting everything together, we conclude
\begin{equation}
\lim_{\sigma \rightarrow 0} I = \frac{\sqrt{b} - \sqrt{a}}{\sqrt{2\pi}} \int_{\partial S} p^\gamma (\sv) d\sv.
\end{equation}

\begin{IEEEbiographynophoto}
{Aamir Anis}
received his Bachelors and Masters of Technology degrees in Electronics and Electrical Communication Engineering from the Indian Institute of Technology (IIT), Kharagpur, India in 2012 and a Ph.D. in Electrical Engineering from the University of Southern California (USC), Los Angeles in 2017. 

He is the recipient of the Best Student Paper award at the ICASSP 2014 conference held in Florence, Italy. His research interests include graph signal processing with applications in machine learning and multimedia compression. He is currently a Software Engineer at Google Inc., Mountain View, California.
\end{IEEEbiographynophoto}

\begin{IEEEbiographynophoto}{Aly El Gamal}
(S '09-M '15) is an Assistant Professor at the Electrical and Computer Engineering Department of Purdue University. He received his Ph.D. degree in Electrical and Computer Engineering and M.S. degree in Mathematics from the University of Illinois at Urbana-Champaign, in 2014 and 2013, respectively. Prior to that, he received the M.S. degree in Electrical Engineering from Nile University and the B.S. degree in Computer Engineering from Cairo University, in 2009 and 2007, respectively. His research interests include information theory and machine learning.

Dr. El Gamal has received a number of awards, including the Purdue Seed for Success Award, the Purdue CNSIP Area Seminal Paper Award, the DARPA Spectrum Challenge (SC2) Contract Award and Phase 1 Top 10 Team Award, and the Huawei Innovation Research Program (HIRP) OPEN Award. He is currenlty a reviewer for the American Mathematical Society (AMS) Mathematical Reviews.
\end{IEEEbiographynophoto}
\begin{IEEEbiographynophoto}{A. Salman Avestimehr} (S'03-M'08-SM'17) is an Associate Professor at the Electrical Engineering Department of University of Southern California. He received his Ph.D. in 2008 and M.S. degree in 2005 in Electrical Engineering and Computer Science, both from the University of California, Berkeley. Prior to that, he obtained his B.S. in Electrical Engineering from Sharif University of Technology in 2003. His research interests include information theory, the theory of communications, and their applications to distributed computing and data analytics.

Dr. Avestimehr has received a number of awards, including the Communications Society and Information Theory Society Joint Paper Award, the Presidential Early Career Award for Scientists and Engineers (PECASE) for ``pushing the frontiers of information theory through its extension to complex wireless information networks'', the Young Investigator Program (YIP) award from the U. S. Air Force Office of Scientific Research, the National Science Foundation CAREER award, and the David J. Sakrison Memorial Prize. He is currently an Associate Editor for the IEEE Transactions on Information Theory.
\end{IEEEbiographynophoto}
\begin{IEEEbiographynophoto}{Antonio  Ortega} (F'07) received the Telecommunications Engineering degree from the Universidad Politecnica de Madrid, Madrid, Spain in 1989 and the Ph.D. in Electrical Engineering from Columbia University, New York, NY in 1994.  In 1994 he joined the Electrical Engineering department at the University of Southern California (USC), where he is currently a Professor and has served as Associate Chair.  He is a Fellow of the IEEE since 2007, and a member of ACM and APSIPA. He has served as associate editor for several IEEE journals, as chair of the Image and Multidimensional Signal Processing (IMDSP) technical committee, and is currently a member of the Board of Governors of the IEEE Signal Processing Society. He was technical program co-chair of ICIP 2008 and PCS 2013.  He was the inaugural Editor-in-Chief of the APSIPA Transactions on Signal and Information Processing. He has received several paper awards, including most recently the 2016 Signal Processing Magazine award and was a plenary speaker at ICIP 2013. His recent research work is focusing on graph signal processing, machine learning, multimedia compression and wireless sensor networks.  Over 40 PhD students have completed their PhD thesis under his supervision at USC and his work has led to over 300 publications in international conferences and journals, as well as several patents.
\end{IEEEbiographynophoto}
\end{document}

%% file: macros.tex
\usepackage{bm}

\long\def\comment#1{}

\newfont{\bbb}{msbm10 scaled 700}

\newfont{\bb}{msbm10 scaled 1100}

\newcommand{\cv}{{\bf c}}

\newcommand{\fv}{{\bf f}}
\newcommand{\gv}{{\bf g}}

\newcommand{\nv}{{\bf n}}

\newcommand{\sv}{{\bf s}}
\newcommand{\tv}{{\bf t}}
\newcommand{\uv}{{\bf u}}

\newcommand{\xv}{{\bf x}}
\newcommand{\yv}{{\bf y}}
\newcommand{\zv}{{\bf z}}
\newcommand{\zerov}{{\bf 0}}
\newcommand{\onev}{{\bf 1}}

\newcommand{\Am}{{\bf A}}
\newcommand{\Bm}{{\bf B}}

\newcommand{\Dm}{{\bf D}}

\newcommand{\Id}{{\bf I}}

\newcommand{\Lm}{{\bf L}}

\newcommand{\Um}{{\bf U}}
\newcommand{\Wm}{{\bf W}}

\newcommand{\Xm}{{\bf X}}

\newcommand{\Mc}{{\cal M}}
\newcommand{\Nc}{{\cal N}}

\newcommand{\Xc}{{\cal X}}

\newtheorem{theorem}{Theorem}
\newtheorem{conjecture}{Conjecture}

\newtheorem{lemma}{Lemma}
\newtheorem{definition}{Definition}

\newtheorem{remark}{Remark}

\newcommand{\E}[1]{\mathbb{E}\left\{{#1}\right\}}
\newcommand{\Var}[1]{\text{Var}\left\{{#1}\right\}}
\newcommand{\Tr}[1]{\text{Tr}\left({#1}\right)}

%% file: main.bbl
\begin{thebibliography}{10}

\bibitem{chapelle_book}
O.~Chapelle, B.~Sch\"{o}lkopf, and A.~Zien, {\em Semi-Supervised Learning
  (Adaptive Computation and Machine Learning)}.
\newblock The MIT Press, 2006.

\bibitem{zhou_nips_04}
D.~Zhou, O.~Bousquet, T.~N. Lal, J.~Weston, and B.~Sch\"{o}lkopf, ``Learning
  with local and global consistency,'' in {\em Advances in Neural Information
  Processing Systems 16} (S.~Thrun, L.~K. Saul, and B.~Sch\"{o}lkopf, eds.),
  pp.~321--328, MIT Press, 2004.

\bibitem{narayanan_06}
H.~Narayanan, M.~Belkin, and P.~Niyogi, ``On the relation between low density
  separation, spectral clustering and graph cuts,'' in {\em Advances in Neural
  Information Processing Systems (NIPS) 19}, 2006.

\bibitem{zhu_03}
X.~Zhu, Z.~Ghahramani, and J.~Lafferty, ``Semi-supervised learning using
  gaussian fields and harmonic functions,'' in {\em IN ICML}, pp.~912--919,
  2003.

\bibitem{trillos_arma_16}
N.~Garc{\'i}a~Trillos and D.~Slep{\v{c}}ev, ``Continuum limit of total
  variation on point clouds,'' {\em Archive for Rational Mechanics and
  Analysis}, vol.~220, pp.~193--241, Apr 2016.

\bibitem{zhou_aistats_11}
X.~Zhou and M.~Belkin, ``Semi-supervised learning by higher order
  regularization,'' in {\em Proceedings of the Fourteenth International
  Conference on Artificial Intelligence and Statistics, {AISTATS} 2011, Fort
  Lauderdale, USA, April 11-13, 2011}, pp.~892--900, 2011.

\bibitem{buhler_icml_09}
T.~B\"{u}hler and M.~Hein, ``Spectral clustering based on the graph
  p-laplacian,'' in {\em Proceedings of the 26th Annual International
  Conference on Machine Learning}, ICML '09, (New York, NY, USA), pp.~81--88,
  ACM, 2009.

\bibitem{elalaoui_colt_16}
A.~E. Alaoui, ``Asymptotic behavior of $\ell_p$-based {L}aplacian
  regularization in semi-supervised learning,'' in {\em Proceedings of the 29th
  Conference on Learning Theory, {COLT} 2016, New York, USA, June 23-26, 2016},
  pp.~879--906, 2016.

\bibitem{belkin_nips_02}
M.~Belkin and P.~Niyogi, ``Using manifold stucture for partially labeled
  classification,'' in {\em Advances in Neural Information Processing Systems
  15} (S.~Becker, S.~Thrun, and K.~Obermayer, eds.), pp.~953--960, MIT Press,
  2003.

\bibitem{belkin_ml_04}
M.~Belkin and P.~Niyogi, ``Semi-supervised learning on {R}iemannian
  manifolds,'' {\em Machine Learning}, vol.~56, no.~1, pp.~209--239, 2004.

\bibitem{shuman_spm_13}
D.~I. Shuman, S.~K. Narang, P.~Frossard, A.~Ortega, and P.~Vandergheynst, ``The
  emerging field of signal processing on graphs: Extending high-dimensional
  data analysis to networks and other irregular domains,'' {\em IEEE Signal
  Processing Magazine}, vol.~30, pp.~83--98, May 2013.

\bibitem{anis_icassp_14}
A.~Anis, A.~Gadde, and A.~Ortega, ``Towards a sampling theorem for signals on
  arbitrary graphs,'' in {\em Acoustics, Speech and Signal Processing (ICASSP),
  2014 IEEE International Conference on}, pp.~3864--3868, May 2014.

\bibitem{shomorony_globalsip_14}
H.~Shomorony and A.~Avestimehr, ``Sampling large data on graphs,'' in {\em
  Signal and Information Processing (GlobalSIP), 2014 IEEE Global Conference
  on}, pp.~933--936, Dec 2014.

\bibitem{chen_tsp_15}
S.~Chen, R.~Varma, A.~Sandryhaila, and J.~Kova\v{c}evi\'{c}, ``Discrete signal
  processing on graphs: Sampling theory,'' {\em Signal Processing, IEEE
  Transactions on}, 2015.

\bibitem{anis_tsp_16}
A.~Anis, A.~Gadde, and A.~Ortega, ``Efficient sampling set selection for
  bandlimited graph signals using graph spectral proxies,'' {\em IEEE
  Transactions on Signal Processing}, vol.~64, pp.~3775--3789, July 2016.

\bibitem{gadde_kdd_14}
A.~Gadde, A.~Anis, and A.~Ortega, ``Active semi-supervised learning using
  sampling theory for graph signals,'' in {\em Proceedings of the 20th ACM
  SIGKDD International Conference on Knowledge Discovery and Data Mining}, KDD
  '14, (New York, NY, USA), pp.~492--501, ACM, 2014.

\bibitem{bousquet_04}
O.~Bousquet, O.~Chapelle, and M.~Hein, ``Measure based regularization,'' in
  {\em Advances in Neural Information Processing Systems (NIPS) 16}, MIT Press,
  2004.

\bibitem{hein_colt_06}
M.~Hein, {\em Uniform Convergence of Adaptive Graph-Based Regularization},
  pp.~50--64.
\newblock Berlin, Heidelberg: Springer Berlin Heidelberg, 2006.

\bibitem{belkin_jcss_08}
M.~Belkin and P.~Niyogi, ``Towards a theoretical foundation for laplacian-based
  manifold methods,'' {\em J. Comput. Syst. Sci.}, vol.~74, no.~8,
  pp.~1289--1308, 2008.

\bibitem{zhou_kdd_11}
X.~Zhou, M.~Belkin, and N.~Srebro, ``An iterated graph laplacian approach for
  ranking on manifolds,'' in {\em Proceedings of the 17th {ACM} {SIGKDD}
  International Conference on Knowledge Discovery and Data Mining, San Diego,
  CA, USA, August 21-24, 2011}, pp.~877--885, 2011.

\bibitem{trillos_arxiv_16}
N.~Garc{\'i}a~Trillos, ``{Variational limits of k-NN graph based functionals on
  data clouds},'' {\em ArXiv e-prints}, July 2016.

\bibitem{slepcev_arxiv_17}
D.~{Slep{\v c}ev} and M.~{Thorpe}, ``{Analysis of $p$-{L}aplacian
  Regularization in Semi-Supervised Learning},'' {\em ArXiv e-prints}, July
  2017.

\bibitem{castro_aap_12}
E.~Arias-Castro, B.~Pelletier, and P.~Pudlo, ``The normalized graph cut and
  cheeger constant: From discrete to continuous,'' {\em Advances in Applied
  Probability}, vol.~44, no.~4, pp.~907--937, 2012.

\bibitem{maier_13}
M.~Maier, U.~von Luxburg, and M.~Hein, ``How the result of graph clustering
  methods depends on the construction of the graph,'' {\em ESAIM: Probability
  and Statistics}, vol.~17, pp.~370--418, 1 2013.

\bibitem{trillos_jmlr_16}
N.~Garc{\'i}a~Trillos, D.~Slep\v{c}ev, J.~von Brecht, T.~Laurent, and
  X.~Bresson, ``Consistency of {C}heeger and ratio graph cuts,'' {\em Journal
  of Machine Learning Research}, vol.~17, no.~181, pp.~1--46, 2016.

\bibitem{anis_15}
A.~Anis, A.~El~Gamal, S.~Avestimehr, and A.~Ortega, ``Asymptotic justification
  of bandlimited interpolation of graph signals for semi-supervised learning,''
  in {\em Acoustics, Speech and Signal Processing (ICASSP), 2014 IEEE
  International Conference on}, April 2015.

\bibitem{narang_icassp13}
S.~Narang, A.~Gadde, and A.~Ortega, ``Signal processing techniques for
  interpolation in graph structured data,'' in {\em Acoustics, Speech and
  Signal Processing (ICASSP), 2013 IEEE International Conference on},
  pp.~5445--5449, May 2013.

\bibitem{narang_globalsip13}
S.~Narang, A.~Gadde, E.~Sanou, and A.~Ortega, ``Localized iterative methods for
  interpolation in graph structured data,'' in {\em Global Conference on Signal
  and Information Processing (GlobalSIP), 2013 IEEE}, pp.~491--494, Dec 2013.

\bibitem{wang_tsp_15}
X.~Wang, P.~Liu, and Y.~Gu, ``Local-set-based graph signal reconstruction,''
  {\em IEEE Transactions on Signal Processing}, vol.~63, pp.~2432--2444, May
  2015.

\bibitem{pesenson_ca_09}
I.~Pesenson, ``Variational splines and paley--wiener spaces oncombinatorial
  graphs,'' {\em Constructive Approximation}, vol.~29, pp.~1--21, Feb 2009.

\bibitem{hein_thesis_06}
M.~Hein, {\em Geometrical aspects of statistical learning theory}.
\newblock PhD thesis, TU Darmstadt, April 2006.

\bibitem{hanneke_jmlr_15}
S.~Hanneke, ``The optimal sample complexity of {PAC} learning,'' {\em Journal
  of Machine Learning Research}, vol.~17, no.~38, pp.~1--15, 2016.

\bibitem{hoeffding_63}
W.~Hoeffding, ``Probability inequalities for sums of bounded random
  variables,'' {\em Journal of the American Statistical Association}, vol.~58,
  no.~301, pp.~13--30, 1963.

\bibitem{billingsley_95}
P.~Billingsley, {\em {Probability and Measure}}.
\newblock New York, NY: Wiley, 3rd~ed., 1995.

\end{thebibliography}
